\documentclass{article}

 \PassOptionsToPackage{authoryear,round,compress}{natbib}
\usepackage[preprint]{neurips_2024}

\usepackage[utf8]{inputenc} 
\usepackage[T1]{fontenc}    
\usepackage[colorlinks=true,linkcolor=blue,citecolor=blue]{hyperref}       
\usepackage{url}            
\usepackage{booktabs}       
\usepackage{amsfonts}       
\usepackage{nicefrac}       
\usepackage{microtype}      
\usepackage{graphicx}
\usepackage{helvet}
\usepackage{courier}
\usepackage{amsmath}
\usepackage{amssymb}
\usepackage{tabularx}
\usepackage{mathtools}
\usepackage{mathabx}
\usepackage{varwidth}
\usepackage{graphicx}
\usepackage{bbm}
\usepackage{bm}
\usepackage{dsfont}
\usepackage{colortbl}
\usepackage[clock]{ifsym}
\usepackage{enumitem}
\usepackage{multicol} 
\usepackage[utf8]{inputenc}
\usepackage{booktabs}  
\usepackage{algorithm}
\usepackage{algorithmic}
\usepackage{wrapfig}
\usepackage{amsthm}
\usepackage{multirow}
\usepackage{subcaption}
\usepackage{setspace}
\usepackage{pifont}
\usepackage{array}
\usepackage[capitalise]{cleveref}
\usepackage{comment}
\usepackage{thm-restate}
\usepackage[dvipsnames]{xcolor}
\usepackage{enumitem}

\definecolor{sexybrown}{HTML}{8B4513}

\definecolor{pierCite}{HTML}{00008B}

\usepackage[toc,page,header]{appendix}
\usepackage{minitoc}

\newcolumntype{C}[1]{>{\centering\arraybackslash}m{#1}}
\newcolumntype{R}[1]{>{\raggedright\arraybackslash}m{#1}}

\usepackage[font=small,skip=10pt]{caption}
\usepackage{setspace}
\DeclareCaptionFormat{myformat}{#1#2#3\hrulefill}

\Crefname{equation}{Eq.}{Eqs.}
\Crefname{figure}{Fig.}{Figs.}
\Crefname{tabular}{Tab.}{Tabs.}
\Crefname{theorem}{Thm.}{Thms.}
\Crefname{lemma}{Lem.}{Lems.}
\Crefname{proposition}{Prop.}{Props.}
\Crefname{definition}{Def.}{Defs.}
\Crefname{algorithm}{Alg.}{Algs.}
\Crefname{corollary}{Corol.}{Corol.}
\Crefname{section}{Sec.}{Sec.}
\Crefname{condition}{Cond.}{Conds.}

\newtheorem{lemma}{Lemma}
\newtheorem{theorem}{Theorem}
\newtheorem{corollary}{Corollary}
\newtheorem{proposition}{Proposition}

\newtheorem{remark}{Remark}

\theoremstyle{definition}
\newtheorem{definition}{Definition}
\newtheorem{example}{Example}

\newcommand{\I}{\mathbbm{1}}
\newcommand{\G}{\mathcal{G}}
\newcommand{\Dc}{\mathcal{D}} 
\newcommand{\M}{\mathcal{M}}

\def\*#1{\mathbf{#1}}
\def\1#1{\mathcal{#1}}
\def\2#1{\mathscr{#1}}
\def\3#1{\mathbb{#1}}

\DeclareBoldMathCommand{\u}{u}
\DeclareBoldMathCommand{\U}{U}
\DeclareBoldMathCommand{\v}{v}
\DeclareBoldMathCommand{\V}{V}
\DeclareBoldMathCommand{\x}{x}
\DeclareBoldMathCommand{\X}{X}
\DeclareBoldMathCommand{\y}{y}
\DeclareBoldMathCommand{\Y}{Y}
\DeclareBoldMathCommand{\z}{z}
\DeclareBoldMathCommand{\Z}{Z}
\DeclareBoldMathCommand{\c}{c}
\DeclareBoldMathCommand{\C}{C}
\DeclareBoldMathCommand{\r}{r}
\DeclareBoldMathCommand{\R}{R}
\DeclareBoldMathCommand{\s}{s}
\DeclareBoldMathCommand{\S}{S}
\DeclareBoldMathCommand{\w}{w}
\DeclareBoldMathCommand{\W}{W}
\DeclareBoldMathCommand{\t}{t}
\DeclareBoldMathCommand{\T}{T}
\DeclareBoldMathCommand{\A}{A}
\DeclareBoldMathCommand{\a}{a}
\DeclareBoldMathCommand{\B}{B}
\DeclareBoldMathCommand{\b}{b}
\DeclareBoldMathCommand{\D}{D}
\DeclareBoldMathCommand{\d}{d}

\newcommand{\PA}{\mathbf{Pa}}

\usepackage{tikz}
\usetikzlibrary{arrows.meta}
\usetikzlibrary{shapes}
\usetikzlibrary{decorations.markings}
\tikzstyle{SCM}=[>={Stealth},
	every node/.style={inner sep=0.25em, transform shape},
	conf-path/.style={<->,dashed,
    	every edge/.append style={in=100,out=80}
    },
    circle-path/.style={draw \circle ->},
    removed/.style={opacity=0.2},
    cut/.style={color=Red!50},
    sel-node/.style={rectangle,inner sep=0.15em, draw, color=cyan, fill=white},
    sig-node/.style={rectangle,inner sep=0.3em, draw, color=Green}
]

\title{Two Layers of Instability in Causal Estimation}

\author{%
  Alexis Bellot$^*$\\
}

\begin{document}
\hypersetup{
    colorlinks,
   linkcolor={pierCite},
    citecolor={pierCite},
    urlcolor={pierCite}
}

\doparttoc 
\faketableofcontents 

\part{} 

\maketitle
\begin{abstract}
There is a precise sense in which drawing causal inferences from observational data is \textit{hard}, even when identifiability is assumed. In particular, \cite{robins1997toward,robins2003uniform} showed that causal effects can be \textit{discontinuous} as a function of the data distribution: two arbitrarily close data distributions might correspond to different causal effects. This is a fact independent of the choice of estimator; however, not all estimators are equally unstable. Our contribution is to surface a second layer of instability that depends on the choice of estimator. We show that many standard point estimates can be read as point summaries of multimodal distributions over the space of structural causal models. As such, estimators can jump discontinuously in the data distribution. This defines a taxonomy of estimators that admits a decision-theoretic reading: stability depends on whether the implicit loss function an estimator optimizes is aligned with the causal effect itself. Specifically, inverse propensity weighted estimators and regression estimators are examples of discontinuous summaries, while explicit posterior means and medians are shown to be continuous.
\end{abstract}

\section{Introduction}
For many practical purposes we want to know what would happen were we to \textit{act} upon a given system. The answers to such questions depend not just on the data distribution, but also on deeper features of underlying data-generating models or mechanisms \citep{pearl2009causality}.

This statement surfaces a structural concern: the causal effect may not be reducible to a functional of the data distribution, in general, unless some additional knowledge such as the (partial) dependency structure encoded in a causal diagram or equivalence class is assumed. In the literature, precise conditions have been derived to determine when such a reduction is possible \citep{tian2002general,bareinboim2022pearl,shpitser2006identification,shpitser2012validity,pearl2022external,huang2008completeness}. On the other hand, there is a statistical concern that asks how to estimate an identified causal effect from finite samples of the data distribution when available, see e.g.\ \citep{jung2021estimating,kennedy2023towards}. A causal effect being identifiable, however, does not necessarily entail well-posed estimation as in general two arbitrarily close distributions $P$ and $Q$ may correspond to different causal effects even when the causal effect is identifiable. This property makes causal estimation \emph{discontinuous} in the data distribution and has important practical implications as with a finite sample of data, we can only know the data distribution up to a tolerance.

This problem has parallels in the nonparametric regression literature \citep{stone1980optimal} but can be significantly more severe in the context of causal estimation. Versions of this result were first shown by \citet{robins1997toward,robins2003uniform} and have been revisited more recently by \citet{schulman2016stability,gordon2021condition,maclaren2019can}. In this context, \citet{aronow2021nonparametric} highlight the advantage that experimental data can confer even when queries are identifiable from observational data (as it leads to a more stable estimation problem).

The contribution we make in this paper is to show that the \textbf{choice of estimator} can add a \emph{second layer} of instability that is independent of the identification map $P \mapsto \Psi(P) := \3E_P[Y_x]$ itself.
We argue that many standard point estimates can be read as $\hat\Psi(P) = T(\Pi_P)$, where $\Pi_P$ is a distribution over SCMs compatible with $P$ (explicit under Bayes, or implicit when plug-in methods select nuisances and complete to an SCM) and $T$ collapses that distribution to a single number.
The stability of the reported effect is then determined by whether $T$ is continuous as a functional on the space of probability measures over SCMs.
This is important in the context of causal inference because with sufficiently expressive model classes, $\Pi_P$ is generically multimodal over observationally compatible SCMs (and thus prone to discontinuities). 
We show that standard plug-in estimators---including propensity-weighted, regression, and doubly robust methods---behave like ``model-selective'' summaries. A canonical instance is $T = \arg\max$ where $T$ selects the model that maximizes a goodness-of-fit criterion. Specifically, they can make the reported effect jump as small changes in $P$ flip the selected model to one associated with a very different causal effect; an estimator can be semiparametrically efficient at the true $P$ while still defining a discontinuous map $P \mapsto T(\Pi_P)$. In contrast, we define a class of ``effect-aligned'' summaries---including the posterior mean, median, and quantiles---that aggregate over $\Pi_P$ and are continuous in $P$.

The contrast between continuous and discontinuous estimators admits a clean decision-theoretic reading.
Summarization instability can be interpreted as a loss-alignment problem: selection losses depend on model identity and do not penalize the magnitude of error in the causal effect, whereas effect-aligned losses do.
For example, $T = \Psi(\arg\max_M \Pi_P(M))$ is Bayes-optimal under indicator loss on the model with $\Psi$ interpreted as the mapping between an SCM and the causal effect functional, while the posterior mean $T = \3E_{\Pi_P}[\Psi]$ is Bayes-optimal under squared-error loss on the effect itself.
It is this alignment that governs continuity in $P$; effect-aligned summarization is a principled, stable response to the ill-conditioning that makes model-selective plug-in reporting hard.

\paragraph{Paper Outline} After introducing the basic framework of our analysis in \Cref{sec:preliminaries}, we start by revisiting the sensitivity of the identification map $P \mapsto \Psi(P)$ in \Cref{sec:extrapolation_gap} to give a self-contained account describing why we expect causal inference to be unstable (Layer 1 of instability). We then introduce a distribution over SCMs that can only be known up to an equivalence set with finite samples of data. With this fact we then introduce a second layer of instability and a taxonomy of estimators depending on their choice of reporting rule in \Cref{sec:bayesian} and \Cref{sec:mode}. We conclude with a decision-theoretic interpretation and discussion of the implications of our results in \Cref{sec:discussion} and \Cref{sec:conclusion}.

\subsection{Preliminaries}
\label{sec:preliminaries}
The basic framework of our analysis rests on \textit{structural causal models} (SCMs) \cite[Def. 7.1.1]{pearl2009causality}, which formalize the notion of a data-generating process. In addition to specifying the distributions of data, these models also specify the generative mechanisms that produce them. For the purpose of causal inference and learning, SCMs provide a broad, fine-grained hypothesis space.

An SCM $M$ is a tuple $\langle \V, \U, \1F, P \rangle$ where $\V$ is a set of endogenous variables and $\U$ is a set of exogenous variables. $\1F$ is a set of functions where each $f_V \in \1F$ decides values of an endogenous variable $V \in \V$ taking as argument a combination of other variables in the system. That is, $V \leftarrow f_{V}(\*\PA_V, \U_V)$ where observed parents $\*\PA_V \subseteq \V$ and unobserved parents $\U_V \subseteq \U$. Drawing values of exogenous variables $\U$ following $P(\U)$ induces the \emph{observational density} over endogenous variables $\V$,
\begin{align}
\label{eq:margin}
    P (\v) = \int_{\Omega_{\U}} \prod_{V \in \V} \I\{f_V(\boldsymbol{pa}_V,\u_V)= v\} dP(\u).
\end{align}
In turn, an intervention $do(x)$ replaces the causal mechanism $f_{X}$ with the assignment $X \leftarrow x$ corresponding to a sub-model $\M_x$. We will use the counterfactual notation $\Psi:=\3E_P[Y_x]$ to denote the average value of the outcome $Y$ in $\M_x$\footnote{We sometimes abuse the $\Psi$ notation to denote the causal functional $M \to \Psi(M)$ as well $P \to \Psi(P)$ when the context is clear.}.

Each SCM $M$ is associated with a \textit{causal diagram} $\G$, where nodes represent endogenous variables $\V$, directed edges represent the arguments $\*\PA_V$ of each function $f_{V}$, and bi-directed edges represent the presence of a common latent variable in the arguments of two functions. Examples are given in \Cref{fig:examples}. We will leverage a special type of clustering of nodes in the graph $\G$ called the \emph{confounded-component} (or $c$-component for short) from \cite{tian2002general}. For a causal graph $\G$, a subset $\C \subseteq \V$ is a $c$-component if any pair $V_i, V_j \in \C$ is connected by a bi-directed path (sequence of bi-directed edges) in $\G$. For example, in the graph in \Cref{fig:examples:a} the (implicit) exogenous variables $\U_Z$, $\U_{XY}$ correspond to $c$-components $\C(\U_Z) = \{Z\}$ and $\C(\U_{XY}) = \{X, Y\}$, respectively. We will use standard graph-theoretic family abbreviations to represent graphical relationships, such as the set of parent nodes of $X$ in $\G$, denoted by $Pa_{\G}(X)$, or the ancestors of $X$ in $\G$, denoted by $An_{\G}(X)$ -- both excluding $X$. For a more detailed survey on SCMs, we refer to \cite{pearl2009causality,bareinboim2022pearl}. All proofs are given in the Appendix.
\begin{figure*}[t]
    \centering
    \hfill\null
    \begin{subfigure}[t]{0.30\linewidth}\centering
      \begin{tikzpicture}[SCM,scale=1]
            \node (X) at (0,0) {$X$};
            \node (Y) at (2,0) {$Y$};
            \node (Z) at (1,1) {$Z$};
    
            \path [->] (X) edge (Y);
            \path [->] (Z) edge (Y);
            \path [->] (Z) edge (X);
      \end{tikzpicture}
    \caption{$\G_1$}
    \label{fig:examples:a}
    \end{subfigure}\hfill
    \begin{subfigure}[t]{0.30\linewidth}\centering
      \begin{tikzpicture}[SCM,scale=1]
            \node (W) at (-1,0) {$W$};
            \node (X) at (0,0) {$X$};
            \node (Y) at (2,0) {$Y$};
            \node (Z) at (1,1) {$Z$};
    
            \path [conf-path] (W) edge[out=70, in=160] (Z);
            \path [->] (W) edge (X);
            \path [->] (X) edge (Y);
            \path [->] (Z) edge (Y);
            \path [conf-path] (X) edge[out=90, in=180] (Z);
        \end{tikzpicture}
    \caption{$\G_2$}
    \label{fig:examples:b}
    \end{subfigure}\hfill
    \begin{subfigure}[t]{0.30\linewidth}\centering
      \begin{tikzpicture}[SCM,scale=1]
        \node (X) at (0,0) {$X$};
        \node (Z) at (1,0) {$Z$};
        \node (Y) at (2,0) {$Y$};

        \path [->] (X) edge (Z);
        \path [->] (Z) edge (Y);
        \path [conf-path] (X) edge[out=45, in=135] (Y);
        \end{tikzpicture}
    \caption{$\G_3$}
    \label{fig:examples:c}
    \end{subfigure}
    \hfill\null
      \caption{Examples of causal diagrams.}
      \label{fig:examples}
\end{figure*}
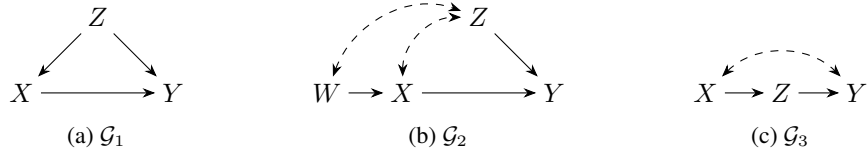

\section{Why is observational causal inference unstable?}
\label{sec:extrapolation_gap}

We say that an inference problem is \textit{unstable} if a tiny perturbation in the data, such as that caused by finite-sample noise, can shift the target (i.e., the implied causal effect) from one value to a radically different one. 

It is well known that certain causal queries are unstable \citep{robins1997toward}. A common example is the average causal effect $\3E_P[Y_x]$ given an observational distribution $P(x, y, z)$ in which $Z$ acts as a confounder for the association between $X$ and $Y$, as in \Cref{fig:examples:a}. In this example,
\begin{align}
    \label{eq:backdoor_bound}
    \sup_{Q:d_{TV}(P,Q) \leq \delta} \Big|\ \3E_P[Y_x] - \3E_Q[Y_x] \ \Big| = \Omega\left( K \cdot \delta \cdot\sup_{z}\left( \frac{1}{P(x \mid z)}\right) \right),
\end{align}
where $K$ is an upper bound on $Y$.

This result shows that the difference in causal effects scales as $\Omega(\cdot)$ of the expression in brackets as $\delta \to 0$. Because the ratio in the lower bound involves an inverse of conditional densities that can evaluate to arbitrarily large values under imperfect overlap (also known as violations of the positivity assumption, see e.g., \citet{hwang2024positivity}), small perturbations in $P$ can cause large changes in the causal effect.

The following theorem shows that this kind of ``instability'' is a property of identifiable causal effects in any causal diagram where $X$ is not connected to any of its children via a bi-directed path\footnote{A lower bound can be derived also for the general case but the resulting expression is less transparent as it depends on the structure of the causal diagram and not only on the topological ordering and $c$-component structure. We defer this result to \Cref{thm:general_ci_stability_bound} in the Appendix.}.

\begin{theorem}
    \label{thm:ci_stability_bound}
    Assume that $X$ is not connected to any of its children via a bi-directed path in $\G_{An(Y)}$ and let $V_1 \prec V_2 \prec \cdots \prec V_k$ be a topological ordering of the variables in $\G_{An(Y)}$ where $V^{(i)} = \{V_1, \dots, V_i\}$. Let $\C$ be the $c$-component in $\G_{An(Y)}$ that contains $X$. Then,
    \begin{align}
        \label{eq:ci_stability_bound}
        \sup_{Q:d_{TV}(P,Q) \leq \delta} \Big|\ \3E_P[Y_x] - \3E_Q[Y_x] \ \Big| = \Omega\left( K \cdot \delta \cdot \sup_{\v\setminus x}\left( \frac{\sum_x \prod_{V_i \in \C} P(v_i \mid v^{(i-1)})}{\prod_{V_i \in \C} P(v_i \mid v^{(i-1)})}\right) \right).
    \end{align}
\end{theorem}

The following example illustrates \Cref{thm:ci_stability_bound}.

\begin{example}
    Given $\G_2$ in \Cref{fig:examples:b} we have that $\C=\{W,X,Z\}$ and $Z \prec W \prec X \prec Y$. Similarly, given $\G_3$ in \Cref{fig:examples:c} we have that $\C=\{X,Y\}$ and $X \prec Z \prec Y$. By \Cref{thm:ci_stability_bound}, the difference in causal effects (as $\delta \to 0$) computed from data generated from these two models is at least of the order of,
    \begin{align}
        \G_2: \ \Omega\left( K \cdot \delta \cdot \sup_{w, z}\left( \frac{1}{P(x \mid w, z)}\right)\right), \qquad \G_3: \ \Omega\left( K \cdot \delta \cdot \sup_{y,z}\left( \frac{\sum_{x'} P(y \mid x', z)P(x')}{P(y \mid x, z)P(x)}\right)\right).
    \end{align}
    We can also verify that for $\G_1$ given in \Cref{fig:examples:a} we recover \Cref{eq:backdoor_bound}.\qed
\end{example}

By inspecting the size of $c$-components in different classes of causal diagrams, we can deduce a \textit{stability hierarchy}. We can show, for instance, that causal inference is more unstable than the corresponding statistical inference problem (especially in the presence of unobserved confounding).

\begin{corollary}[Stability Hierarchy]
    \label{cor:stability_hierarchy}
    Using the conditions of \Cref{thm:ci_stability_bound}, causal effects $\3E_P[Y_x]$ in semi-Markov models are more unstable than causal effects in Markov models, which are more unstable than the corresponding conditional expectation $\3E_P[Y \mid x]$.
\end{corollary}

To see this we can use Cauchy--Schwarz and Jensen's inequalities to derive an ordering over the instability factors of the three types of models,

\begin{align}
    \underbrace{\sup_{\v\setminus x}\left( \frac{\sum_x \prod_{V_i \in \C} P(v_i \mid v^{(i-1)})}{\prod_{V_i \in \C} P(v_i \mid v^{(i-1)})}\right)}_{\text{Causal query in semi-Markov model}} 
    \ \geq \
    \underbrace{\sup_{pa_x}\left(\frac{1}{P(x \mid pa_X)} \right)}_{\text{Causal query in Markov model}} 
    \ \geq 
    \underbrace{\frac{1}{P(x)}}_{\text{Statistical query}}.
    \label{eq:stability_hierarchy_ordering}
\end{align}

\begin{wrapfigure}{r}{0.5\textwidth}
    \centering
    \includegraphics[scale=0.4 ]{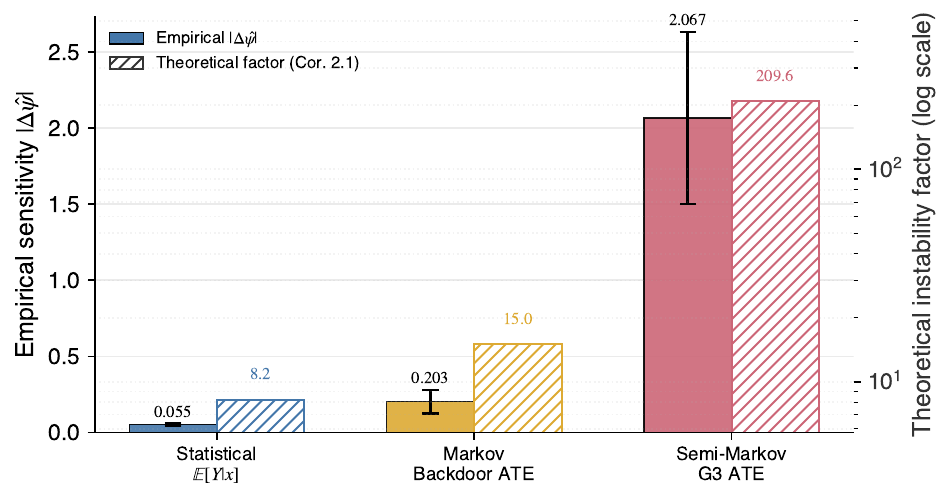}
    {\captionsetup{font=footnotesize}%
    \caption{Stability Hierarchy: \Cref{cor:stability_hierarchy,ex:stability_hierarchy_numerical}.}%
    \label{fig:stability_hierarchy}}
\end{wrapfigure}
The first term on the l.h.s.\ is the general instability factor of \Cref{thm:ci_stability_bound}, the middle term is the instability factor we would obtain by adjusting for the parents of $X$ in a Markov model (i.e., with no unobserved confounding), and the last term is the instability factor we obtain for the corresponding statistical query.


\begin{example}[Stability Hierarchy]
    \label{ex:stability_hierarchy_numerical}
    \Cref{fig:stability_hierarchy} summarizes a numerical simulation aligned with the three tiers above. We consider binary $X$, discrete $Z$, and a binned outcome $Y$, and evaluate three plug-in targets on population tabular distributions: the statistical contrast $\3E_P[Y \mid x=1] - \3E_P[Y \mid x=0]$; the Markov ATE on $\G_1$ (\Cref{fig:examples:a}); and the semi-Markov ATE on $\G_3$ (\Cref{fig:examples:c}). For each of $35$ settings with varying overlap, we apply a fixed total-variation shift of size $\delta = 0.025$: a perturbation of $P(X)$ for the statistical query, and approximate worst-case directional shifts of the joint table for the causal queries. \Cref{fig:stability_hierarchy} reports the resulting mean plug-in sensitivities $|\Delta \hat{\psi}|$ (solid bars) alongside the corresponding mean instability factors from \Cref{cor:stability_hierarchy} (hatched bars, log scale). Both increase from the statistical to the Markov to the semi-Markov tier.\qed
\end{example}

These results concern Layer~1---the map $P \mapsto \Psi(P)$---before any reporting rule is chosen.
\Cref{sec:bayesian} introduces the posterior over SCMs $\Pi_P$ and shows that a \emph{second}, summarization-layer instability can arise from the choice of $T$, even when $\Psi(P)$ is identified.
\section{Uncertainty over SCMs and summarization}
\label{sec:bayesian}

This section introduces Layer~2.
Reporting a single number requires summarizing uncertainty over compatible SCMs via an action functional $T$, so that $\hat\Psi(P) = T(\Pi_P)$.
The stability of the \emph{reported} effect depends on whether $T$ is continuous as a functional on $\1P(\3M)$\footnote{The notation $\1P(\3M)$ denotes the set of Borel probability measures on $\3M$.}---a property that can fail independently of Layer~1.


In practice, the population data distribution $P$ that identifies the causal effect $\Psi$ is known up to finite sample approximation. The set of models that might have generated a finite sample of data is wider,
\begin{align}
    \label{eq:robust_set}
    \1R := \{M\in\3M: d_{\1W}(P_M, P_n) \leq \varepsilon_n\}
\end{align}
where $P_n$ is the empirical distribution over the data domain from $n$ i.i.d. samples from $P$, for some tolerance parameter $\varepsilon_n > 0$.

On this set, the causal effect can be highly uncertain as \Cref{thm:ci_stability_bound} shows.

\subsection{Posterior over SCMs}

A principled way to represent this structural uncertainty is via a posterior on a space of candidate SCMs. For a fixed set $\V$, the space of SCMs can be defined to be a product space of continuous functions $\{f_V:V\in\V\}$ between variables and of joint probability measures over the domain of latent variables $\U$. We define the topology of $\3M$ as the product of the strong topology (uniform convergence) on the space of structural mechanisms and the weak topology on the space of exogenous measures. We define the prior $\Pi_0$ as a Borel probability measure on $(\3M,\1B(\3M))$. For a given probability distribution $P$ over the domain of $\V$, the posterior measure $\Pi_P$ is given by:
\begin{align}
    d\Pi_P = \frac{\1L_P(M)}{Z_P}d\Pi_0, \quad Z_P = \int \1L_P(M) d\Pi_0,
\end{align}
where $Z_P$ is the marginal likelihood. 

A canonical choice for the likelihood is the Gibbs posterior $\1L_P(M) = \exp\{-n\cdot \ell_n(M;P)\}$ where $\ell_n(M;P)$ is a loss functional for the fit of SCM $M$ to data $P$. For example, for quantifying uncertainty over observational distributions implied by SCMs, $\ell_n(M;P) = d_{\1W}(P, P_M)$ where $d_{\1W}$ is the Wasserstein distance. $n$ (often interpreted as the ``effective sample size'') serves as a precision parameter governing the posterior concentration rate $\delta_n$. $\delta_n$ is defined as the smallest sequence of numbers such that the posterior probability of a ``ball'' around the manifold of compatible models $\{M\in\3M: P_M=P\}$ goes to 1 as $n \to \infty$, i.e.,
\begin{align}
    \Pi_{P} \big( M : d_{\1W}(P_M, P) > C \delta_n \big) \to 0.
\end{align}

Under standard regularity, this construction makes $\3M$ a Polish space and the observational map $M \mapsto P_M$ continuous (Appendix~\ref{sec:posterior_stability_app}).
Two SCMs that are close in $\3M$ therefore generate similar data and, by \Cref{prop:continuity_causal_functional}, similar causal effects.

\begin{proposition}[Continuity of the Causal Functional]
    \label{prop:continuity_causal_functional}
    Let $\3M$ be endowed with the product topology as defined above. Then the causal functional $\Psi: \3M \to \3R$ defined by $\3E_{P_M}[Y_x]$ is continuous on $\3M$.
\end{proposition}

The next theorem establishes that the data-to-posterior map $P \mapsto \Pi_P$ is continuous in $d_{TV}$, with a sensitivity controlled by the posterior concentration rate.

\begin{theorem}[Posterior Stability]
    \label{thm:posterior_stability}
    Let $P, Q$ be two distributions over the data such that $d_{TV}(P, Q) \leq \epsilon$, with $n\epsilon \to 0$ as $\epsilon \to 0$.
    Then, under some regularity conditions and for a well-defined prior $\Pi_0$,
    \begin{align}
        d_{TV}(\Pi_P, \Pi_Q) \leq C\cdot n \cdot \delta_n \cdot \epsilon,
    \end{align}
    where $n$ is the effective sample size in the Gibbs likelihood, $\delta_n$ is the posterior contraction rate, and $C$ is an absolute constant absorbing $K$.
\end{theorem}

In particular, the ``shape'' of the causal uncertainty is protected against small shifts in the data distribution. The posterior $\Pi_P$ together with the causal functional $\Psi$ induces a probability measure over the causal effect itself, which is the central object analyzed in the rest of the paper.

\subsection{Point estimators as summaries}
\label{sec:summaries}

This section introduces the central object of analysis in the rest of the paper: the causal posterior measure and the point estimators that summarize it.

\begin{definition}[Causal Posterior Measure]
    \label{def:causal_posterior}
    Let $\Psi: \3M \to \3R$ be defined by $\Psi(M) = \3E_{P_M}[Y_x]$. For any Borel set $B\in\1B(\3R)$, the pushforward $\Psi_\#\Pi_P$ is the posterior measure over the causal effect given observational distribution $P$:
    \begin{align}
        \Psi_\# \Pi_P(B) := \Pi_P\left(\{M\in\3M: \Psi(M)\in B\}\right).
    \end{align}
\end{definition}

 
\begin{definition}[Point Estimator]
    \label{def:point_estimator}
    A point estimator of the causal effect is defined as $\hat\Psi(P) = T\big(\Pi_P\big)$, where $T$ is an action functional that collapses a probability measure over SCMs to a real number.
\end{definition}

We next formalize two classes of summaries that arise in practice.

\begin{definition}[Effect-aligned summary]
    \label{def:effect_aligned}
    A point estimator is effect-aligned if $T(\Pi_P) \;=\; E(\Psi_\#\Pi_P)$ where $E:\1P(\3R)\to\3R$ is a functional that is continuous in the weak topology on $\1P(\3R)$.
\end{definition}

\begin{definition}[Model-selective summary]
    \label{def:model_selective}
    A point estimator is model-selective if $T=\Psi\circ S$ where $S:\1P(\3M)\to\3M$ with $S(\Pi)\in\arg\max_{M\in\3M}\Pi(M)$ is a selection functional.
\end{definition}

Examples of effect-aligned summaries include the posterior mean, medians, and quantiles. Estimators in practice need not compute the full Gibbs posterior $\Pi_P$ over nonparametric $\3M$. Examples of model-selective summaries include the posterior mode, plug-in estimators, and more generally any rule that reports the causal effect under a single selected SCM (\Cref{sec:mode}).

Continuity of the data-to-estimator map $P \mapsto \hat\Psi(P) = T(\Pi_P)$\footnote{We will often write $\hat\Psi(P)$ and $T(\Pi_P)$ interchangeably when denoting estimators of causal effects and the context is clear.} factors through the posterior. When $T$ is effect-aligned (\Cref{def:effect_aligned}), we have $T(\Pi) = E(\Psi_\#\Pi)$ with $E$ weakly continuous on $\1P(\3R)$; continuity then follows from \Cref{thm:posterior_stability}, the pushforward map $\Pi \mapsto \Psi_\#\Pi$ (a direct consequence of \Cref{prop:continuity_causal_functional}), and weak continuity of $E$. 

The next two sections study the continuity of effect-aligned and model-selective summaries $T$.

\begin{figure}[t]
    \centering
    \includegraphics[width=\linewidth]{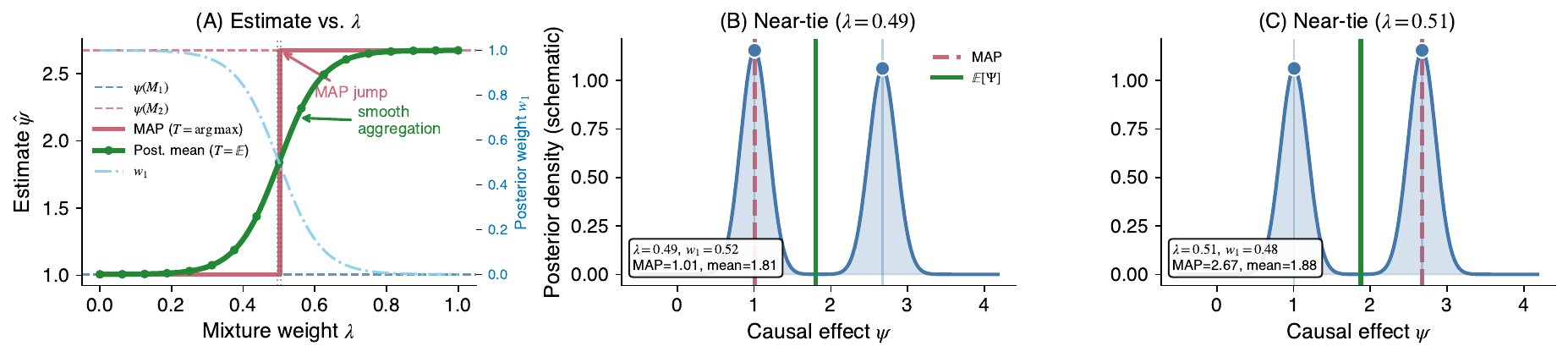}
    \caption{Numerical illustration of \Cref{ex:map_instability,ex:map_instability_continued} with a pair of SCMs with near-equal fit to the data.
    We follow the setup of \Cref{ex:map_instability} with $A=5$, $k=25$, $n=55$.
    Panel~(A) shows the model-selective summary (posterior mode) estimate jumping near $\lambda=\tfrac12$ while the posterior mean varies smoothly in $\lambda$.
    Panel~(B) shows a schematic pushforward $\Psi_\#\Pi_P$ at the near-tied mixture $\lambda=0.49$.
    Panel~(C) shows the posterior mean shifting only slightly at $\lambda=0.51$ when a small tilt past the tie flips the posterior mode to $\psi(M_2)$.}
    \label{fig:ex41_combined}
\end{figure}

\section{Model-selective summaries}
\label{sec:mode}

This section shows that many common estimators of causal effects can be interpreted as model-selective summaries over restricted hypothesis classes. Such estimators define discontinuous maps $P \mapsto \hat\Psi(P)$ even when the causal effect is point-identified.

For any $M_1, M_2 \in \1R$ (\Cref{eq:robust_set}), assuming a Gibbs likelihood, we have that $\bigl|\1L_{P_n}(M_1) - \1L_{P_n}(M_2)\bigr| \le 1 - e^{-n\varepsilon_n}$. In particular, when $n\varepsilon_n$ is small the Gibbs likelihood is \emph{approximately constant} on $\1R$. This implies also that for a sufficiently uniform prior over $\1R$, the posterior is \emph{approximately constant} on $\1R$ given the data,
\begin{align}
    \Pi_{P_n}(M_1) \approx \Pi_{P_n}(M_2), \forall M_1, M_2 \in \1R,
\end{align}
while the causal effect implied by these SCMs might be different.

Model-selective summaries are unstable because many SCMs can have near-maximal posterior mass yet imply different causal effects, so that even small changes in the input data can switch the selected model and jump the reported effect.
This is the model-selective counterpart to \Cref{prop:stable_summaries}.

\begin{theorem}[Discontinuity of model-selective summaries]
\label{thm:model_selective_discontinuity}
    Let $T$ be model-selective in the sense of \Cref{def:model_selective}.
    Let $\3M$ be a nonparametric hypothesis class of SCMs over continuous variables with $X\in An_{\G}(Y)$.
    Under imperfect overlap, there exist observational distributions $P$ and sequences $Q_n \to P$ in $d_{TV}$ such that
    \[
        \limsup_{n\to\infty} \ \bigl| \ T(\Pi_{Q_n}) - T(\Pi_P) \ \bigr| \geq \Delta
    \]
    for some $\Delta > 0$.
    In particular, the map $P \mapsto T(\Pi_P)$ is not continuous at such $P$.
\end{theorem}

The following example illustrates \Cref{thm:model_selective_discontinuity}.

\begin{example}[Instability of model-selective summaries]
\label{ex:map_instability}
    Let $P(z, x, y)$ be a given joint observational density. We will consider two SCMs $M_1$ and $M_2$ that differ in the mechanism of $Y$. In $M_1$ it is given by $f_Y(x, z, u)$ and in $M_2$ it is given by $f_Y(x, z, u) + A \cdot \exp\left( -k \cdot p(x, z) \right)$. For large values of $k$ the observational densities under these two models can be made arbitrarily close since $\exp\left( -k \cdot p(x, z) \right) \to 0$ as $k\to\infty$ whenever $p(x, z) \gg 0$. However, 
    \begin{align}
        \Psi(M_2) = \Psi(M_1) + A \cdot \underbrace{\int \exp\left( -k \cdot p(x, z) \right) p(z) dz}_{\text{Fragility Weight } \alpha}.
    \end{align}
    The causal effects induced by these two models differ by a constant term. Unless $x$ has perfect overlap, i.e., $p(x, z) > \delta$ for all $z$ in the population, there will be a region of $Z$ where $p(x, z)$ is small. In those regions, $\exp(-k \cdot p(x, z)) \approx 1$. If that ``lack of overlap'' region has non-zero mass $\alpha$ under the marginal $p(z)$ then,
    $$|\Psi(M_2) - \Psi(M_1)| \approx A \cdot \alpha$$
    Since $A$ can be arbitrarily large in a non-parametric class, the causal effects can be made radically different. With small perturbations as $P$ changes from $P_{M_1}$ to $P_{M_2}$ the posterior-mode summary transports from $\Psi(M_1)$ to $\Psi(M_2)$. A numerical illustration of this phenomenon is shown in \Cref{fig:ex41_combined}.\qed
\end{example}

The following examples highlight the selection rule governing inverse propensity weighting and regression estimators.

\begin{example}[IPW as model-selective summary]
\label{ex:ipw-map}
Assuming a data-generating process compatible with the causal diagram in \Cref{fig:examples:a}, the IPW estimator can be interpreted as selecting an SCM by maximum likelihood on the propensity score,
$$\Psi_{\text{IPW}} = \3E[Y; M_x], \quad \text{ where } \quad M: P_M(x\mid z) = \hat \omega(x \mid z) \in  \arg\max_{\omega} \ \text{Criterion}(\omega).$$
$M$ is selected by fitting the propensity score to the data (e.g. by maximum likelihood) and selecting $M$ such that the induced probability of treatment $P_M(x\mid z)$ is given by the estimated propensity score $\hat \omega(x \mid z)$. This process defines an implicit posterior with loss $\ell_n(M)=-\frac{1}{n}\sum_i \log P_M(X_i\mid Z_i)$ and a flat prior and reports its mode.
\qed
\end{example}

\begin{example}[Outcome regression as model-selective summary]
Let $\1F_\mu$ be a parametric class of outcome models $\mu_\theta(x,z)$. Under the same setting as \Cref{ex:ipw-map}, the outcome regression estimator can be interpreted as selecting an SCM by maximum likelihood on the outcome model,
$$\Psi_{\text{REG}} = \3E[Y; M_x], \quad \text{ where } \quad M: \3E_{P_M}[Y|x,z] = \hat{\mu}(x,z) \in  \arg\max_{\mu \in \1F_\mu} \ \text{Criterion}(\mu).$$
Similarly, $M$ is selected by fitting the outcome model to the data (e.g. by minimizing mean squared error) and selecting $M$ such that the conditional expectation of the outcome is given by the estimated outcome model $\hat{\mu}(x,z)$ for all $x,z$.\qed
\end{example}

In these examples, the selection of the SCM is implicit. Plug-in estimators select nuisances by maximum likelihood on a restricted class and complete to an SCM whose observational distribution matches the data.
This yields an implicit Gibbs posterior $d\Pi_P \propto \exp\{-n \cdot\ell_n(M, P)\}d\Pi_0$ over a restricted subfamily. Moreover, both estimators satisfy the discontinuity conclusion of \Cref{thm:model_selective_discontinuity}.

\section{Effect-aligned summaries: $T = \3E$}
\label{sec:mean}

By \Cref{def:effect_aligned}, an effect-aligned summary has the form $T(\Pi_P) = E(\Psi_\#\Pi_P)$ with $E$ weakly continuous on $\1P(\3R)$.

\begin{proposition}[Stable summaries]
    \label{prop:stable_summaries}
    Let $T$ be effect-aligned in the sense of \Cref{def:effect_aligned}, and assume $|\Psi(M)| \leq K$ for all $M \in \3M$.
    Under the conditions of \Cref{thm:posterior_stability}, if $d_{TV}(P, Q) \leq \epsilon$ with $nD\epsilon \to 0$, then
    \begin{align}
        \bigl|T(\Pi_P) - T(\Pi_Q)\bigr| \to 0 \quad \text{as } \epsilon \to 0.
    \end{align}
    In particular, $P \mapsto T(\Pi_P)$ is continuous in $P$ at fixed $n$, $\delta_n$, and prior hyperparameters.
\end{proposition}

Any effect-aligned summary---including the posterior mean, median, and quantiles---inherits stability in $P$. 
The corollary and theorem below specialize this general principle to $T = \3E$; the variance-sensitivity bound is a sharpening specific to the mean.

\begin{corollary}
    \label{cor:bound_posterior_mean}
    Let $\Psi: \3M \to \3R$ be defined by $\Psi(M) = \3E_{P_M}[Y_x]$. Under the conditions of \Cref{thm:posterior_stability},
    \begin{align}
        \label{eq:bound_posterior_mean}
        \Big| \ \mathbb{E}_{\Pi_P}[\Psi] - \mathbb{E}_{\Pi_Q}[\Psi] \ \Big| \leq 2K \cdot C\cdot n \cdot \delta_n \cdot \epsilon,
    \end{align}
    provided that the target $Y$ is bounded by $K$.
\end{corollary}

This is the instance of \Cref{prop:stable_summaries} for $T = \3E$, with an explicit Lipschitz constant.

This bound is structurally different from the bound in \Cref{thm:ci_stability_bound} as it does not carry the overlap factor. Its sensitivity constant depends on the posterior contraction rate $\delta_n$ (a property of the model class and prior), not on local features of $P$ near singular regions.

The term $n\cdot\delta_n$ reflects a tension between the narrowing of the posterior variance and the growing likelihood sensitivity in Bayesian estimators. For small $n$, we get wide posteriors but low likelihood sensitivity in $P$ (recall $\ln L_P(M) = - n \cdot d_{\1W}(P, P_M)$). For larger $n$, we get narrow posteriors but more likelihood sensitivity. The following theorem shows that this term can be refined to approximately $n \cdot \delta_n^2$ once we account for the posterior variance of $\Psi$, which is what makes the bound genuinely small in well-contracted regimes.

\begin{theorem}[Posterior Variance-Sensitivity Bound]
    \label{thm:posterior_variance_sensitivity}
    Let $\Psi: \3M \to \3R$ be defined by $\Psi(M) = \3E_{P_M}[Y_x]$, and let $P, Q$ be two distributions over the data such that $d_{TV}(P, Q) \leq \epsilon$. Then, 
    \begin{align}
        \Big| \ \mathbb{E}_{\Pi_P}[\Psi] - \mathbb{E}_{\Pi_Q}[\Psi] \ \Big| \leq \sqrt{\text{Var}_{\Pi_P}(\Psi(M))} \cdot \sqrt{\text{Var}_{\Pi_P}(S(M))} + o(\epsilon),
    \end{align}
    where $S(M) = \ln L_P(M) - \ln L_Q(M)$ is the relative log-likelihood (score function).
\end{theorem}

The bounds in \Cref{cor:bound_posterior_mean,thm:posterior_variance_sensitivity} can be related as we increase the sample size $n$ (under some regularity conditions, see \Cref{lemma:variance_score}). For a perturbation of size $\epsilon$ in $P$ we can show that,
\begin{align}
    \sqrt{\text{Var}_{\Pi_P}(\Psi(M))} \approx \delta_n, \qquad \sqrt{\text{Var}_{\Pi_P}(S(M))} \approx n \cdot \delta_n \cdot \epsilon.
\end{align}

\Cref{cor:bound_posterior_mean} is therefore a global worst-case bound that assumes the posterior does not contract, i.e., the standard deviation equals the maximum possible value $K$. In general we can expect $\delta_n$ to shrink with increasing sample size $n$ at a rate that depends on the complexity of $\3M$. 




The next example continues the discussion of the instability of the model-selective estimator in \Cref{ex:map_instability}, showing how the posterior mean changes as we transition between two nearby data distributions.

\begin{example}[\Cref{ex:map_instability} continued]
\label{ex:map_instability_continued}
    Using the same two $M_1$ and $M_2$ as in \Cref{ex:map_instability}, as $k \to \infty$ both models are approximately equally likely under the posterior as both give an equally good fit to the input data (for a reasonable prior). We might imagine that for input data $P=P_{M_1}$ the posterior will assign mass $(0.5 + \eta)$ to $M_1$ and $(0.5 - \eta)$ to $M_2$. And similarly, given input data $P=P_{M_2}$ the posterior will assign mass $(0.5 + \eta)$ to $M_2$ and $(0.5 - \eta)$ to $M_1$. The posterior mean averages over the causal effects induced by the two likely SCMs and the maximum deviation as we transition from $P_{M_1}$ to $P_{M_2}$ is roughly,
    $$\Big| \ \mathbb{E}_{\Pi_{P_{M_1}}}[\Psi] - \mathbb{E}_{\Pi_{P_{M_2}}}[\Psi] \ \Big| \approx 2 \eta \to 0 \quad \text{as} \quad k \to \infty.$$
    The posterior mean is stable as $k \to \infty$, in contrast to the posterior mode. A numerical illustration of this phenomenon is shown in \Cref{fig:ex41_combined}.\qed
\end{example}


\section{Summarization taxonomy and caveats}
\label{sec:discussion}

We now make the underlying loss-alignment mechanism explicit and collect the taxonomy.

Each summary can be cast as the Bayes action under a posterior expected loss.
Model-selective estimators minimize the indicator loss on the \emph{identity} of the model,
\begin{align}
    \widehat{M} \in \arg\min_{M \in \3M} \3E_{\Pi_P}\bigl[\I\{M \neq M^*\}\bigr],
\end{align}
where $M^*$ is the true SCM\footnote{Equivalently, $\arg\max_M \Pi_P(M)$, which by Bayes' rule equals $\arg\max_M \1L_P(M)\Pi_0(M)$.}.
The posterior mean $\3E_{\Pi_P}[\Psi]$ minimizes squared-error loss on the \emph{effect itself},
\begin{align}
    \arg\min_{\hat\Psi \in \3R} \3E_{\Pi_P}\bigl[(\hat\Psi - \Psi(M))^2\bigr].
\end{align}
The first loss does not depend on the value of $\Psi(M)$ at all: a small data change that flips the mode from $M_1$ to $M_2$ may shift the estimate by an arbitrarily large amount $|\Psi(M_1) - \Psi(M_2)|$, while the model-space loss only records that we picked a ``wrong'' model.
The second loss is convex and continuous in the action and respects $|\Psi(M_1) - \Psi(M_2)|$ by construction.
Stability, in this view, is a consequence of the fact that the loss is aligned with the quantity one is estimating.

This reading tells us that the dichotomy is not specific to $\arg\max$ versus $\3E$.
\emph{Any} effect-aligned functional $T$ inherits the stability of \Cref{prop:stable_summaries} via $\hat\Psi(P) = T(\Pi_P)$.
Conversely, any decision rule whose loss is invariant under arbitrary changes in $\Psi(M)$---mode on model space, MLE, $\arg\max$ over a parametric submodel---inherits the discontinuity of \Cref{thm:model_selective_discontinuity}.
The rift is between losses that live on the effect and losses that live on the model.

\subsection{Is effect-aligned summarization more stable?}

The claim is justified in the sense that effect-aligned summaries such as $T = \3E$ are Lipschitz in $\epsilon$ for fixed $n$, $\delta_n$, and prior hyperparameters (\Cref{cor:bound_posterior_mean,prop:stable_summaries}), while model-selective summaries are not (\Cref{thm:model_selective_discontinuity}).

However, several caveats are worth surfacing. 

Firstly, model-selective and effect-aligned summaries typically target different estimands rather than being different estimators for the same estimand.
As such they strike different bias-variance tradeoffs and our discussion can be equally framed within the classical theory of regularization of ill-posed inverse problems \citep{evans2002inverse,maclaren2019can}.
Whether one estimand is preferred over the other depends on our relative tolerance for bias and variance.
For example, $T = \3E$ trades off potentially infinite variance in non-parametric model spaces for non-zero bias, while some instances of model-selective reporting are unbiased but involve larger variance.

Secondly, in many practical settings, the difference between selection strategies collapses.
For example, in a strongly-convex log-likelihood, identifiable, parametric regime, the posterior is asymptotically Gaussian and unimodal, and mode and mean concentrate on the same value at rate $n^{-1/2}$.
The setting that is problematic is the nonparametric, multimodal regime that arises under imperfect overlap and rich SCM classes.
The taxonomy this paper analyzes is thus informative in classes broad enough that one cannot rule out multimodality \emph{a priori}. 

Thirdly, our results study the question of stability as robustness to noise, not consistency.
That is, we are interested in bounding how much the estimate moves under small perturbations of $P$; but do not bound how far the estimate is from the true effect (which connects to the bias-variance tradeoff).
A misspecified prior or restricted hypothesis class produces a stably wrong estimator.
We restrict to point-identified targets throughout; the same $T$-taxonomy applies whenever one reports a point summary of SCM uncertainty.

Finally, given that we quantify uncertainty over the causal effect through a posterior distribution over SCMs, uncertainty from instability and uncertainty from under-identification of the causal effect can both be expressed with the same framework. An explicitly Bayesian approach (if a reasonable prior can be defined) has the potential to honestly quantify a wide range of inherent difficulties in the process of drawing causal inferences from observational data.  



\section{Conclusion}
\label{sec:conclusion}
Causal inference from observational data can be unstable even when a causal effect is point-identified.
A first layer of difficulty is functional ill-conditioning: two distributions that are arbitrarily close may imply different effects \citep{robins1997toward,robins2003uniform,maclaren2019can}.
This is a property of the identification map itself, before any reporting rule or estimator is chosen. We consider a second layer of difficulty, defined as ``summarization instability'', which is a consequence of collapsing a multimodal distribution over the space of structural causal models to a single number. Many standard point estimates can be read as such summaries. Among them we define a class of ``model-selective'' summaries that can jump discontinuously in $P$ when compatible SCMs tie, independently of the discontinuity in the identification map. These include the posterior mode, plug-in estimators, and more generally any rule that reports the causal effect under a single selected SCM. On the other hand, there is a class of ``effect-aligned'' summaries that includes the posterior mean, median, and quantiles, shown to be continuous in $P$ under regularity conditions. This taxonomy can be explained through a decision-theoretic reading that shows that stability depends on whether the implicit loss of the estimator is aligned with model identity or with the causal effect.

\bibliography{bibliography}
\bibliographystyle{plainnat}

\newpage

\appendix

\hypersetup{
    colorlinks=true,
    linkcolor=sexybrown,
    filecolor=sexybrown,
    urlcolor=sexybrown,
}
\addcontentsline{toc}{section}{Appendix} 
\part{Appendix} 
\parttoc 
\hypersetup{
    colorlinks,
   linkcolor={pierCite},
    citecolor={pierCite},
    urlcolor={pierCite}
}

\newpage

\section{Related Work}
\label{sec:app_related_work}

A large body of work studies when causal quantities can be expressed as functionals of the observational (or experimental) data distribution, given graphical or algebraic assumptions about the data-generating process \citep{tian2002general,shpitser2006identification,shpitser2012validity,huang2008completeness,bareinboim2022pearl,pearl2022external}. Complementary to these structural results is a statistical literature on how to estimate identified effects from finite samples, including semiparametric efficiency theory, double machine learning, and doubly robust estimation \citep{robins1997toward,jung2021estimating,kennedy2023towards}. Our focus is somewhat orthogonal: we take identifiability as given and ask when estimation from the data distribution is well-posed in the sense of continuity in $P$. As the introduction notes, identifiable causal functionals need not be continuous in $P$; closeness of distributions does not imply closeness of effects. That gap between identification and stable estimation is the starting point for the present work.

The discontinuity of causal estimation has been documented in several settings. \citet{robins2003uniform} show that in sufficiently rich DAG model classes there need not exist uniformly consistent estimators or valid confidence intervals for causal effects, even when effects are identifiable. \citet{robins1997toward} and \citet{stone1980optimal} connect related difficulties to curse-of-dimensionality phenomena and optimal nonparametric rates. \cite{schulman2016stability,gordon2021condition} more recently make the instability explicit as a sensitivity of the identification map and study this in detail in the context of the ID algorithm. In light of this and other results, \citet{aronow2021nonparametric} argue for the use of randomized data and show that this leads to stable estimators as you avoid extrapolation in regions that are not covered by the data. The tension between uniqueness (identifiability) and continuity (stable estimability) is also well known in classical inverse problems and robust statistics. Hadamard's notion of a well-posed problem requires existence, uniqueness, and stability of the solution; identifiability corresponds to injectivity of a forward map, which does not guarantee a continuous inverse. \citet{maclaren2019can} develop this perspective for causal inference explicitly, distinguishing ``what can be identified'' from ``what can be estimated'' and relating intrinsic barriers to estimability to ill-posedness. \citet{evans2002inverse} and \citet{horowitz2014ill} review parallel ideas in statistics and econometrics, where irregular identification and ill-posed inverse problems arise routinely. Regularization---penalizing complexity of solutions or aggregating over a set of compatible models---is the standard remedy for ill-posedness \citep{tikhonov1977solutions}. We connect to that tradition in \Cref{sec:discussion}: effect-aligned summaries can be read as regularized targets (bias--variance tradeoff) relative to model-selective plug-in rules.

Related work on partial identification \citep{manski1990nonparametric,balke1994bounds,zhang2022partial,bellot2023towards,bellot2024estimating} studies bounds and posterior uncertainty when point identification fails.
Our focus is complementary: we take point identification as given and ask how different point summaries $T(\Pi_P)$ behave under perturbations of $P$, but the underlying ambiguity is the same: a finite sample only determines the data distribution up to tolerance, leaving many SCMs compatible with the evidence that imply different causal effects. An explicitly Bayesian approach, propagating the uncertainty in the SCM generating the data quantifies both under-identifiability and instability in identifiable settings.

Bayesian approaches to causal inference already maintain a distribution over compatible models, mechanisms, or identified sets.
Partial-identification methods represent ambiguity via posterior mass over effect ranges \citep{zhang2022partial,jalaldoust2024partial}; flexible outcome models integrate over uncertain response surfaces \citep{hill2011bayesian,hahn2020bayesian}; and structure-learning pipelines maintain posteriors over graphs or equivalence classes \citep{claassen2012bayesian}.
Our focus is on the stability profiles of point summaries of those posteriors under perturbations of the observational distribution $P$.
Many Bayesian workflows are only partially aggregative at the reporting stage.
A pipeline that is Bayesian over graphs or mechanisms but reports the MAP structure and a plug-in effect remains model-selective and inherits the discontinuities of \Cref{thm:model_selective_discontinuity}, even when $\Pi_P$ itself is computed exactly.
Conversely, effect-aligned summaries---posterior means, medians, quantiles, and other weakly continuous functionals on $\1P(\3R)$---inherit continuity from \Cref{prop:stable_summaries}.
Flexible Bayesian estimators such as BART are stable insofar as they report expectations over the causal functional; instability re-enters when implementation collapses to selecting one fitted mechanism or hyperparameter configuration \citep{hill2011bayesian,hahn2020bayesian}.

\newpage
\section{Proofs}

This section develops and proves the theoretical results presented in the main body of this paper.

\subsection{Proofs of statements in \Cref{sec:extrapolation_gap}}
\label{app:proof_lower_bound}

\textbf{\Cref{thm:ci_stability_bound} restated.}
Assume that $X$ is not connected to any of its children via a bi-directed path in $\G_{An(Y)}$ and $V_1 \prec V_2 \prec \cdots \prec V_k$ be a topological sort of the variables in $\G_{An(Y)}$ where $V^{(i)} = \{V_1, \dots, V_i\}$. Let $\C$ be the $c$-component in $\G_{An(Y)}$ that contains $X$. Then,
\begin{align}
    \sup_{Q:d_{TV}(P,Q) \leq \delta} \Big|\ \3E_P[Y_x] - \3E_Q[Y_x] \ \Big| = \Omega\left( K \cdot \delta \cdot \sup_{\v\setminus x}\left( \frac{\sum_x \prod_{V_i \in \C} P(v_i \mid v^{(i-1)})}{\prod_{V_i \in \C} P(v_i \mid v^{(i-1)})}\right) \right).
\end{align}

\begin{proof}
    We give a self-contained proof of the lower bound in \Cref{thm:ci_stability_bound} that constructs a single admissible perturbation $Q$, of total-variation distance exactly $\delta$ from $P$, achieving the claimed order. 

    \textbf{Setup.} We will leverage a special type of clustering of nodes in the graph $\G$ called the \emph{confounded-component} (or $c$-component for short) from \cite{tian2002general}. Recall that for a causal graph $\G$, a subset $\C \subseteq \V$ is a $c$-component if any pair $V_i, V_j \in \C$ is connected by a bi-directed path (sequence of bi-directed edges) in $\G$. Let $\3C(\G)$ denote the set of $c$-components in $\G$. Let $\3Q[\C] = P(\c \mid do(\v \setminus \c))$.

    By marginalizing over the observed variables,
    \begin{align}
        \ P(y_x)= \sum_{y} \ y \ \sum_{\v \setminus y} \3Q[\V \setminus X].
    \end{align}
    Let $\C$ be the $c$-component in $\G_{An(Y)}$ that contains $X$. By using the standard decomposition of $C$-factors, we have that,
    \begin{align}
        \3Q[\V \setminus X] = \3Q[\C \setminus X] \prod_{\Dc \in \3C(\G) \setminus \C} \3Q[\Dc], \qquad \3Q[\V] = \3Q[\C] \prod_{\Dc \in \3C(\G) \setminus \C} \3Q[\Dc]
    \end{align}
    
    $\C\setminus X$ is ancestral in $\G_{\C}$, meaning that $\C\setminus X$ includes all of its ancestors in $\G_{\C}$ since we have assumed that $X$ is not connected to any of its children via a bi-directed path. By Lemma~4 in \citet{tian2002general},
    $$\3Q[\C\setminus X] = \sum_{x} \ \3Q[\C].$$
    
    Further, for $V_1 \prec V_2 \prec \cdots \prec V_k$ a topological ordering of the variables in $\G_{An(Y)}$ where $V^{(i)} = \{V_1, \dots, V_i\}$, Lemma~4 in \citet{tian2002general} gives
    \begin{align}
        \3Q[\C] = \prod_{V_i\in \C} P(v_i \mid v^{(i-1)}).
    \end{align}

    Our causal effect of interest can then be written as,
    \begin{align}
        \3E_P[Y_x] &= \sum_{v} \ y \ \frac{\3Q[\C \setminus X]}{\3Q[\C]} P(\v)\\
        &= \sum_{v} \ y \ \frac{\sum_x \prod_{V_i\in \C} P(v_i \mid v^{(i-1)})}{\prod_{V_i\in \C} P(v_i \mid v^{(i-1)})} P(\v)\\
        &= \3E_P[w(\V)\cdot \I\{X=x\} \cdot Y],
    \end{align}
    where $w(\v) = \frac{\sum_x \prod_{V_i\in \C} P(v_i \mid v^{(i-1)})}{\prod_{V_i\in \C} P(v_i \mid v^{(i-1)})}$.

    For any distribution $R$ over $\V$ Markov to $\G$ and under which $\3E_R[Y_x]$ is identifiable,
    \begin{align}
        \3E_R[Y_x] \;=\; \3E_R\bigl[\,w_R(\V)\,\I\{X=x\}\,Y\,\bigr], \qquad w_R(\v) \;=\; \frac{\sum_x \prod_{V_i\in\C} R(v_i \mid v^{(i-1)})}{\prod_{V_i\in\C} R(v_i \mid v^{(i-1)})}.
    \end{align}
    We will construct $Q$ such that
    \begin{align}
        \label{eq:clean_lb:preserve}
        Q(v_i \mid v^{(i-1)}) \;=\; P(v_i \mid v^{(i-1)}) \qquad \text{for every } V_i \in \C \text{ and every } v^{(i-1)} \text{ in the joint support,}
    \end{align}
    which is equivalent to $\3Q_Q[\C] = \3Q_P[\C]$ and hence to $w_Q \equiv w_P$. Then,
    \begin{align}
        \label{eq:clean_lb:linear_form}
        \3E_P[Y_x] - \3E_Q[Y_x] 
        \;=\; \int_{\Omega_\V} w_P(\v)\,\I\{X=x\}\,Y(\v)\,\bigl(p(\v) - q(\v)\bigr)\,d\v.
    \end{align}
    Because the construction below modifies a single conditional in the topological factorization of $P$ and leaves all others untouched, $Q$ is automatically a probability measure Markov to $\G$, and the identification hypotheses of \Cref{thm:ci_stability_bound} (which depend on $\G$ alone) transfer to $Q$.

    \textbf{The function $w_P$ and the perturbation point.} Set $k^* := \max\{i : V_i \in \C\}$. Inspection of $w_P$ shows that it depends on $\v$ only through $v^{(k^*)}$, so we regard $w_P$ as a function on $\Omega_{V^{(k^*)}}$. Fix $\eta \in (0, 1/2)$ and choose a point $v^{(k^*)*}$ with $X^* = x$ at which
    \begin{align}
        w_P(v^{(k^*)*}) \;\geq\; (1-\eta)\,\sup_{\v\setminus x}\, w_P(\v).
    \end{align}
    By continuity of $w_P$ (or, in the discrete case, restriction to a single atom of positive mass), choose a measurable neighborhood $U^*$ of $v^{(k^*)*}$ in $\Omega_{V^{(k^*)}}$ on which $X = x$ identically and
    \begin{align}
        w_P(v^{(k^*)}) \;\geq\; (1-2\eta)\,\sup_{\v\setminus x}\, w_P(\v), \qquad \beta \;:=\; P(V^{(k^*)} \in U^*) \;>\; 0.
    \end{align}

    \textbf{Outcome variable and oscillation assumption.} We assume $Y \notin \C$; the case $Y \in \C$ is discussed in \Cref{rem:clean_lb:Y_in_C} below. Then $Y$ appears after every variable in $\C$ in the topological order; without loss of generality, $Y = V_k$. Write $V^{(k-1)}$ for the remaining variables.

    We further assume that there exist two measurable sets $S^+, S^- \subseteq \Omega_Y$ and constants $c_Y \in (0, K]$, $\alpha_0 \in (0, 1/2]$, such that
    \begin{align}
        \inf_{y \in S^+}\, y \;-\; \sup_{y \in S^-}\, y \;\geq\; 2c_Y,
    \end{align}
    and $P(Y \in S^{\pm} \mid v^{(k-1)}) \geq \alpha_0$ for every $v^{(k-1)}$ with $v^{(k^*)} \in U^*$. This is a mild positive-mass/spread assumption on the conditional law of $Y$ near the maximizer of $w_P$; the resulting lower bound is $\Omega(c_Y\,\delta\,\sup w_P)$, equivalently $\Omega(K\,\delta\,\sup w_P)$ when $c_Y$ is a constant fraction of $K$.

    \textbf{Construction of $Q$.} Fix $\delta \in (0,\, \beta\,\alpha_0)$ and set $\alpha := \delta/\beta \in (0, \alpha_0)$. Define $Q$ by leaving every topological factor of $P$ unchanged except $P(Y \mid V^{(k-1)})$, which is modified on $U^*$ by moving conditional mass $\alpha$ from $S^-$ to $S^+$:
    \begin{align}
        q(Y \mid v^{(k-1)}) \;=\;
        \begin{cases}
            p(Y\mid v^{(k-1)})\cdot\bigl(1 + \alpha/P(S^+ \mid v^{(k-1)})\bigr), & v^{(k^*)} \in U^*,\; Y \in S^+,\\[2pt]
            p(Y\mid v^{(k-1)})\cdot\bigl(1 - \alpha/P(S^- \mid v^{(k-1)})\bigr), & v^{(k^*)} \in U^*,\; Y \in S^-,\\[2pt]
            p(Y\mid v^{(k-1)}), & \text{otherwise}.
        \end{cases}
    \end{align}
    The conditional $q(\cdot \mid v^{(k-1)})$ is a valid probability density (non-negativity uses $\alpha \leq \alpha_0 \leq P(S^- \mid v^{(k-1)})$, and the total mass is unchanged since $\alpha = +\alpha$ added on $S^+$ cancels $-\alpha$ removed from $S^-$). Condition \eqref{eq:clean_lb:preserve} holds since no conditional of any $V_i \in \C$ is altered. Hence $Q$ is Markov to $\G$ and $w_Q \equiv w_P$.

    \textbf{Total-variation distance.} The marginal of $V^{(k-1)}$ is unchanged, so $p(\v) - q(\v) = P(v^{(k-1)})\,\bigl(p(Y\mid v^{(k-1)}) - q(Y\mid v^{(k-1)})\bigr)$, and
    \begin{align}
        d_{TV}(P, Q)
        &= \tfrac{1}{2}\int P(v^{(k-1)}) \int_{\Omega_Y}\bigl| p(Y\mid v^{(k-1)}) - q(Y\mid v^{(k-1)})\bigr|\,dY\,dv^{(k-1)}\\
        &= \int_{\{v^{(k-1)}\,:\,v^{(k^*)} \in U^*\}} P(v^{(k-1)})\cdot \alpha\;dv^{(k-1)}
        \;=\; \alpha\beta \;=\; \delta,
    \end{align}
    where we used $d_{TV}(q(\cdot\mid v^{(k-1)}),\, p(\cdot\mid v^{(k-1)})) = \alpha$ on $U^*$ (mass $\alpha$ moved between the disjoint sets $S^+, S^-$) and $0$ off $U^*$.

    \textbf{Mean shift of $Y$ on $U^*$.} For $v^{(k-1)}$ with $v^{(k^*)} \in U^*$,
    \begin{align}
        \3E_Q[Y \mid v^{(k-1)}] - \3E_P[Y \mid v^{(k-1)}]
        &= \alpha\,\bigl(\3E_P[Y \mid Y \in S^+, v^{(k-1)}] - \3E_P[Y \mid Y \in S^-, v^{(k-1)}]\bigr)\\
        &\geq \alpha\,\bigl(\inf_{S^+} y - \sup_{S^-} y\bigr) \;\geq\; 2\alpha\, c_Y.
    \end{align}

    \textbf{Lower bound on the causal effect difference.} Since $w_P(\v)\,\I\{X=x\}$ depends only on $v^{(k^*)} \subseteq v^{(k-1)}$, it factors out of the conditional integral over $Y$ in \eqref{eq:clean_lb:linear_form}:
    \begin{align}
        \3E_Q[Y_x] - \3E_P[Y_x]
        &= \int P(v^{(k-1)})\, w_P(v^{(k^*)})\, \I\{X=x\}\,\Bigl(\3E_Q[Y\mid v^{(k-1)}] - \3E_P[Y\mid v^{(k-1)}]\Bigr)\,dv^{(k-1)}\\
        &\geq \int_{\{v^{(k-1)}\,:\,v^{(k^*)} \in U^*\}} P(v^{(k-1)})\cdot(1-2\eta)\sup_{\v\setminus x} w_P \cdot 2\alpha\, c_Y\;dv^{(k-1)}\\
        &= 2\,c_Y\,(1-2\eta)\,\alpha\,\beta\,\sup_{\v\setminus x} w_P
        \;=\; 2\,c_Y\,(1-2\eta)\,\delta\,\sup_{\v\setminus x} w_P.
    \end{align}
    Absorbing $2\,c_Y\,(1-2\eta)$ into the constant of $\Omega$, this gives
    \begin{align}
        \sup_{Q:\, d_{TV}(P,Q) \leq \delta} \bigl|\3E_P[Y_x] - \3E_Q[Y_x]\bigr|
        \;=\; \Omega\!\left( K\cdot\delta\cdot\sup_{\v\setminus x}\frac{\sum_x \prod_{V_i\in\C} P(v_i\mid v^{(i-1)})}{\prod_{V_i\in\C} P(v_i\mid v^{(i-1)})}\right),
    \end{align}
    under the standing oscillation assumption that $c_Y$ is a constant fraction of $K$. This is the same lower bound asserted in \Cref{thm:ci_stability_bound}.


    \begin{remark}[The case $Y \in \C$]
    \label{rem:clean_lb:Y_in_C}
    If $Y \in \C$, the conditional $P(Y \mid V^{(k-1)})$ is constrained by \eqref{eq:clean_lb:preserve}, so the construction above does not apply directly. Two adjustments recover the bound. \emph{(i)} If there exists any variable $V_j \notin \C$ with $j > k^*$ whose value materially influences $\3E_P[Y \mid V^{(j-1)}, V_j]$, then moving conditional mass in $P(V_j \mid V^{(j-1)})$ instead of in $P(Y \mid V^{(k-1)})$ produces the same mean shift in $Y$ up to a constant, and the proof goes through unchanged. \emph{(ii)} If $\V = \C$ (no variables outside the $c$-component containing $X$), then every $Q$ satisfying \eqref{eq:clean_lb:preserve} agrees with $P$ on the full joint, and no $w_Q \equiv w_P$ perturbation produces any change in the effect. In that degenerate case the lower bound must be recovered via a non-trivial $w_Q \neq w_P$ construction; this can be made rigorous by a first-order (Fr\'echet) expansion of $R \mapsto \3E_R[Y_x]$ at $P$, whose influence function is precisely $w_P\,\I\{X=x\}\,Y$ (up to centering), giving the same $\Omega(K\,\delta\,\sup w_P)$ order to leading order in $\delta$.
    \end{remark}
\end{proof}

\textbf{\Cref{cor:stability_hierarchy} restated.} Using the conditions of \Cref{thm:ci_stability_bound}, the instability factors appearing in the lower bounds satisfy
    \begin{align}
        \sup_{\v\setminus x}\left( \frac{\sum_x \prod_{V_i \in \C} P(v_i \mid v^{(i-1)})}{\prod_{V_i \in \C} P(v_i \mid v^{(i-1)})}\right)
        \ \geq \
        \sup_{pa_X}\left(\frac{1}{P(x \mid pa_X)} \right)
        \ \geq \
        \frac{1}{P(x)}.
    \end{align}
    In particular, the lower-bound prefactor for a semi-Markov causal query is at least as large as that for the corresponding Markov backdoor adjustment, which is in turn at least as large as that for the statistical query $\3E_P[Y \mid x]$.

\begin{proof}
    To show this we use Cauchy--Schwarz and Jensen's inequalities to derive an ordering over the instability factors of the three models. Our goal is to show that,
    \[
        \underbrace{\sup_{\v\setminus x}\left( \frac{\sum_x \prod_{V_i \in \C} P(v_i \mid v^{(i-1)})}{\prod_{V_i \in \C} P(v_i \mid v^{(i-1)})}\right)}_{\text{Causal query in semi-Markov model}} 
        \ \geq \
        \underbrace{\sup_{pa_x}\left(\frac{1}{P(x \mid pa_X)} \right)}_{\text{Causal query in Markov model}} 
        \ \geq 
        \underbrace{\frac{1}{P(x)}}_{\text{Statistical query}}.
    \]
    The first term on the l.h.s.\ is the general instability factor of \Cref{thm:ci_stability_bound}, the middle term is the instability factor we would obtain by adjusting for the parents of $X$ in a Markov model, and the last term is the instability factor we obtain for the corresponding statistical query.

    Consider the lower bound of \Cref{thm:ci_stability_bound},
    \begin{align}
        \sup_{Q:d_{TV}(P,Q) \leq \delta} \Big|\ \3E_P[Y_x] - \3E_Q[Y_x] \ \Big| = \Omega\left( K \cdot \delta \cdot \sup_{\v\setminus x}\left( \frac{\sum_x \prod_{V_i \in \C} P(v_i \mid v^{(i-1)})}{\prod_{V_i \in \C} P(v_i \mid v^{(i-1)})}\right) \right).
    \end{align}
    Let,
    \begin{align}
        R = \frac{\sum_x \prod_{V_i \in \C} P(v_i \mid v^{(i-1)})}{\prod_{V_i \in \C} P(v_i \mid v^{(i-1)})}, \quad A = \sum_x \prod_{V_i \in \C} P(v_i \mid v^{(i-1)}),
    \end{align}
    where $A=A(\c\setminus x)$ is the numerator of the ratio $R$ and is interpreted as a function of $\C\setminus X$ with all other variables fixed. A
    is a valid probability distribution over $\C\setminus X$ since,
    \begin{align}
        \sum_{\c\setminus x} A(\c\setminus x) = \sum_{\c\setminus x}\sum_x \prod_{V_i \in \C} P(v_i \mid v^{(i-1)}) = 1.
    \end{align}
    We then consider calculating the expectation of the ratio $R$ over the distribution $A$,
    \begin{align}
        \3E_A[R] = \sum_{\c\setminus x} A(\c\setminus x) \cdot R = \frac{1}{P(x \mid v^{(i-1)})}\sum_{\c\setminus x} A(\c\setminus x) \frac{A(\c\setminus x)}{\prod_{V_i \in \C\setminus X} P(v_i \mid v^{(i-1)})}.
    \end{align}
    Note that the denominator of the ratio is evaluated for a fixed $x$, the intervention value. By the Cauchy--Schwarz inequality (specifically Titu's lemma, $\sum \frac{a_i^2}{b_i} \geq \frac{(\sum a_i)^2}{\sum b_i}$), we have that,
    \begin{align}
        \3E_A[R] &= \frac{1}{P(x \mid v^{(i-1)})}\sum_{\c\setminus x} A(\c\setminus x) \frac{A(\c\setminus x)}{\prod_{V_i \in \C\setminus X} P(v_i \mid v^{(i-1)})}\\
        & \geq \frac{1}{P(x \mid v^{(i-1)})} \frac{\left(\sum_{\c\setminus x} A(\c\setminus x)\right)^2}{\sum_{\c\setminus x} \prod_{V_i \in \C\setminus X} P(v_i \mid v^{(i-1)})}\\
        & = \frac{1}{P(x \mid v^{(i-1)})}.
    \end{align}
    We can then conclude that,
    \begin{align}
        \sup_{\v\setminus x}\left( \frac{\sum_x \prod_{V_i \in \C} P(v_i \mid v^{(i-1)})}{\prod_{V_i \in \C} P(v_i \mid v^{(i-1)})}\right) \geq \3E_A[R] \geq \sup_{v^{(i-1)}}\frac{1}{P(x \mid v^{(i-1)})}.
    \end{align}
    We can then arrive at the result by using Jensen's inequality and properties of suprema. Let $W \in V^{(i-1)}$ be in the conditioning set of the density above. Then,
    \begin{align}
        \3E\left[\frac{1}{P(x \mid v^{(i-1)} \setminus w, w)}\right] \geq \frac{1}{\3E\left[P(x \mid v^{(i-1)} \setminus w, w)\right]} = \frac{1}{ P(x \mid v^{(i-1)} \setminus w) },
    \end{align}
    where the expectation is taken over the distribution $P(w \mid v^{(i-1)} \setminus w)$. This holds for any instantiation of the conditioning set $W$. Therefore,
    \begin{align}
        \sup_{v^{(i-1)}}\frac{1}{P(x \mid v^{(i-1)})} \geq \sup_{v^{(i-1)} \setminus w}\frac{1}{ P(x \mid v^{(i-1)} \setminus w) }.
    \end{align}
    By generalizing this relationship we then conclude that,
    \[
        \sup_{\v\setminus x}\left( \frac{\sum_x \prod_{V_i \in \C} P(v_i \mid v^{(i-1)})}{\prod_{V_i \in \C} P(v_i \mid v^{(i-1)})}\right)
        \ \geq \
        \sup_{pa_x}\left(\frac{1}{P(x \mid pa_X)} \right)
        \ \geq 
        \frac{1}{P(x)},
    \]
    showing the claim.

\end{proof}

\begin{theorem}
    \label{thm:general_ci_stability_bound}
    Let $\G$ be a causal diagram with $X\in An_{\G}(Y)$, $P$ be an observational distribution over $\V$, and $\3E_P[Y_x]$ an identifiable causal effect. Let $\Dc = An(Y)_{\G_{\overline X}}$ and $V_1 \prec V_2 \prec \cdots \prec V_k$ be a topological sort of the variables in $\G$ where $V^{(i)} = \{V_1, \dots, V_i\}$. Then,
    \begin{align}
        \label{eq:general_ci_stability_bound}
        \sup_{Q:d_{TV}(P,Q) \leq \delta} \Big|\ \3E_P[Y_x] - \3E_Q[Y_x] \ \Big| = \Omega\left( K \cdot \delta \cdot \sup_{\v\setminus x}\left( \frac{P(\d \setminus x \mid do(x))}{\prod_{V_i \in \Dc} P(v_i \mid v^{(i-1)})}\right) \right).
    \end{align}
\end{theorem}

\begin{proof}
    This statement generalizes \Cref{thm:ci_stability_bound} to the case where $X$ is possibly connected to any of its children via a bi-directed path in $\G_{An(Y)}$. \cite{pearl2009causality} showed that the causal effect of $X$ on $Y$ may be identified in such cases; consider for instance the Napkin graph example.

    The proof follows exactly from the proof of \Cref{thm:ci_stability_bound} by using the fact that $\3Q[\Dc\setminus X] = P(\d \setminus x \mid do(x))$ instead of the reduction $\3Q[\Dc\setminus X] =\sum_x \prod_{V_i\in \Dc} P(v_i \mid v^{(i-1)})$ used in the proof of \Cref{thm:ci_stability_bound}. Here, $P(\d \setminus x \mid do(x))$ is identifiable by the assumption that $\3E_P[Y_x]$ is identifiable but its closed form expression in terms of $P$ will depend on the causal structure of the graph.
\end{proof}

\subsection{Proofs of statements in \Cref{sec:bayesian}}

\begin{lemma}[Continuity of the Observational Map]
    The mapping $\Gamma: \mathcal{M} \to \mathcal{P}(\mathcal{V})$ defined by $P_M = F_{\#} P$ (the push-forward of the latent measure $P(\u)$ through the structural equations) is continuous.
\end{lemma}

\begin{proof}
    Let $(F_k, P_k) \to (F, P)$. We show that for any bounded continuous test function $h \in C_b(\mathcal{V})$, the integral $\int h(v) dP_{M_k}(v) \to \int h(v) dP_M(v)$.
    Using the change of variables formula:
    $$\left| \int h(F_k(u)) dP_k(u) - \int h(F(u)) dP(u) \right|$$
    Adding and subtracting $\int h(F(u)) dP_k(u)$:
    $$\leq \underbrace{\int |h(F_k(u)) - h(F(u))| dP_k(u)}_{(1)} + \underbrace{\left| \int h(F(u)) dP_k(u) - \int h(F(u)) dP(u) \right|}_{(2)}$$
    
    For $(1)$: Since $F_k \to F$ uniformly and $h$ is continuous on a compact set (hence uniformly continuous), $h \circ F_k \to h \circ F$ uniformly. Thus, the integral vanishes as $k \to \infty$.
    
    For $(2)$: Since $h \circ F$ is a continuous bounded function on $\mathcal{U}$, and $P_k \to P$ weakly, this integral vanishes by the definition of weak convergence.
    
    Therefore $\Gamma$ is continuous, which implies the Likelihood $L_P(M) = \exp(-n \ d_{\1W}(P, P_M))$ is continuous on $\mathcal{M}$.
\end{proof}

\textbf{\Cref{prop:continuity_causal_functional} restated}. Let $\3M$ be endowed with the product topology as defined in the main body of this paper. Then the causal functional $\Psi: \3M \to \3R$ defined by $\3E_{P_M}[Y_x]$ is continuous on $\3M$.

\begin{proof}
    To prove that $\Psi(M) = \mathbb{E}_M[Y | do(x)]$ is a continuous functional on the joint space $\mathcal{M} = \mathcal{F} \times \mathcal{P}(\mathcal{U})$, we show that for any sequence of models $M_k = (F_k, P_k)$ converging to $M = (F, P)$ in the product topology, the causal effects converge: $\Psi(M_k) \to \Psi(M)$.

    Under the intervention $do(X=x)$, the value of $Y$ is determined by the composition of the structural functions along the causal paths from $X$ to $Y$, using the latent noise $u$. Let us denote this composed function as $\mathcal{G}_F(x, u)$.
    $$\Psi(F, P) = \int_{\mathcal{U}} \mathcal{G}_F(x, u) dP(u)$$
        
    Let $(F_k, P_k) \to (F, P)$ in $\mathcal{M}$. This implies:
    1.  $F_k \to F$ uniformly (in the supremum norm).
    2.  $P_k \to P$ weakly (in the Wasserstein sense).
    
    We want to show that $|\Psi(F_k, P_k) - \Psi(F, P)| \to 0$. We use the standard triangle inequality decomposition:
    \begin{align}
    &\left| \ \int \mathcal{G}_{F_k} dP_k - \int \mathcal{G}_{F} dP \ \right| \\ 
    &\leq \underbrace{\int \left| \mathcal{G}_{F_k}(x, u) - \mathcal{G}_{F}(x, u) \right| dP_k(u)}_{\text{Term 1}} + \underbrace{\left| \int \mathcal{G}_{F}(x, u) dP_k(u) - \int \mathcal{G}_{F}(x, u) dP(u) \right|}_{\text{Term 2}}
    \end{align}
    
    \textbf{Term 1}. Because $\mathcal{G}_F$ is a composition of continuous functions on a compact domain, and each $f_{i,k}$ converges uniformly to $f_i$, the composed function $\mathcal{G}_{F_k}$ converges uniformly to $\mathcal{G}_F$. Therefore, $\sup_u | \mathcal{G}_{F_k}(x, u) - \mathcal{G}_F(x, u) | \leq \epsilon_k$, where $\epsilon_k \to 0$. Since $P_k$ is a probability measure (total mass 1), $(1)$ is bounded by $\epsilon_k \cdot 1$, which vanishes.
    
    \textbf{Term 2}. This term represents the standard definition of Weak Convergence. The function $\mathcal{G}_F(x, u)$ is a continuous function of $u$ (it is a chain of continuous structural equations). By the definition of weak convergence, if $P_k \to P$ weakly, then for any continuous bounded test function $h(u)$, the integral $\int h dP_k \to \int h dP$. Setting $h(u) = \mathcal{G}_F(x, u)$, Term 2 vanishes as $k \to \infty$.    
    
    Since both terms go to zero, $\Psi(M)$ is a continuous functional on $\mathcal{M}$.
\end{proof}

\subsubsection{Posterior Stability}
\label{sec:posterior_stability_app}

\paragraph{Setting and notation.}
Throughout this section we work under the assumption that the support of the endogenous variables $\V$ has bounded diameter $D = \mathrm{diam}(\Omega) < \infty$. This is a mild assumption -- it is in any case implicit in the bounded-outcome hypothesis $\sup_M |\Psi(M)| \leq K$ used in the corollary below -- and it lets us bridge between the two metrics on probability measures that appear in the paper. Specifically, under bounded support,
\begin{align}
    \label{eq:W1_TV}
    d_{\1W}(P, Q) \leq D \cdot d_{TV}(P, Q),
\end{align}
for all probability measures $P, Q$ on $\Omega$. We use \Cref{eq:W1_TV} to lift bounds that arise naturally in $W_1$ (because the Gibbs likelihood is defined via the Wasserstein distance in \Cref{sec:bayesian}) to the $d_{TV}$-based stability statements that match the framework of \Cref{thm:ci_stability_bound}. All multiplicative constants in the bounds below absorb $D$.

We prove the result via a sequence of lemmas. The first translates a $d_{TV}$-bound on the posterior measures into a bound on the posterior mean of any bounded functional.

\begin{lemma}[Expectation Difference Bound]
    \label{lemma:exp_diff_bound}
    Let $K \geq 0$ be such that $\sup_M |\Psi(M)| \leq K$. Then
    \begin{align}
        \Big| \ \mathbb{E}_{\Pi_P}[\Psi] - \mathbb{E}_{\Pi_Q}[\Psi] \ \Big| \leq 2K \cdot d_{TV}(\Pi_P, \Pi_Q).
    \end{align}
\end{lemma}

\begin{proof}
    Let $\pi_P$ and $\pi_Q$ denote the Radon--Nikodym derivatives of $\Pi_P$ and $\Pi_Q$ with respect to the prior $\Pi_0$. Then
    \begin{align}
        \Big| \ \mathbb{E}_{\Pi_P}[\Psi] - \mathbb{E}_{\Pi_Q}[\Psi] \ \Big| 
        = \left| \int_{\3M} \Psi(M) (\pi_P(M) - \pi_Q(M))\, d\Pi_0(M) \right|.
    \end{align}
    By H\"older's inequality with $p = \infty, q = 1$,
    \begin{align}
        \Big| \ \mathbb{E}_{\Pi_P}[\Psi] - \mathbb{E}_{\Pi_Q}[\Psi] \ \Big| 
        \leq \sup_{M} |\Psi(M)| \cdot \int_{\3M} | \pi_P(M) - \pi_Q(M) |\, d\Pi_0(M).
    \end{align}
    By the definition of total variation distance, $d_{TV}(\Pi_P, \Pi_Q) = \tfrac{1}{2} \int |\pi_P - \pi_Q|\, d\Pi_0$, and so
    \begin{align}
        \Big| \ \mathbb{E}_{\Pi_P}[\Psi] - \mathbb{E}_{\Pi_Q}[\Psi] \ \Big| 
        \leq K \cdot 2\, d_{TV}(\Pi_P, \Pi_Q) = 2K \cdot d_{TV}(\Pi_P, \Pi_Q).
    \end{align}
\end{proof}

\begin{lemma}[Metric Sensitivity of the Likelihood]
    \label{lemma:met_sens_lik}
    For any $M \in \3M$, if $d_{TV}(P, Q) \leq \epsilon$, then $|\ln L_P(M) - \ln L_Q(M)| \leq n D \epsilon$.
\end{lemma}

\begin{proof}
    The Gibbs log-likelihood is $\ln L_P(M) = -n\, d_{\1W}(P, P_M)$, so
    \begin{align}
        |\ln L_P(M) - \ln L_Q(M)| = n \cdot |d_{\1W}(Q, P_M) - d_{\1W}(P, P_M)|.
    \end{align}
    By the reverse triangle inequality applied to the metric $d_{\1W}$,
    \begin{align}
        d_{\1W}(Q, P_M) &\leq d_{\1W}(P, Q) + d_{\1W}(P, P_M),\\
        d_{\1W}(P, P_M) &\leq d_{\1W}(P, Q) + d_{\1W}(Q, P_M),
    \end{align}
    hence $|d_{\1W}(P, P_M) - d_{\1W}(Q, P_M)| \leq d_{\1W}(P, Q)$. Combining with \Cref{eq:W1_TV} and the hypothesis,
    \begin{align}
        d_{\1W}(P, Q) \leq D \cdot d_{TV}(P, Q) \leq D \epsilon,
    \end{align}
    and therefore $|\ln L_P(M) - \ln L_Q(M)| \leq n D \epsilon$.
\end{proof}
    
\begin{lemma}[Stability of the Evidence Ratio]
    \label{lemma:stab_ev_ratio}
    $e^{-n D \epsilon} \leq \frac{Z_Q}{Z_P} \leq e^{n D \epsilon}$.
\end{lemma}

\begin{proof}
    Write $Z_Q / Z_P$ as an expectation over $\Pi_P$:
    \begin{align}
        \frac{Z_Q}{Z_P} = \int \frac{L_Q(M)}{L_P(M)} \cdot \frac{L_P(M)}{Z_P}\, d\Pi_0(M) = \mathbb{E}_{\Pi_P} \left[ \frac{L_Q(M)}{L_P(M)} \right].
    \end{align}
    By \Cref{lemma:met_sens_lik}, $\exp(-n D \epsilon) \leq L_Q(M)/L_P(M) \leq \exp(n D \epsilon)$, and the expectation of a random variable lying in $[a, b]$ also lies in $[a, b]$:
    \begin{align}
        e^{-n D \epsilon} \leq \mathbb{E}_{\Pi_P} \left[ \frac{L_Q(M)}{L_P(M)} \right] \leq e^{n D \epsilon}.
    \end{align}
\end{proof}

\begin{lemma}[Score Variance]
    \label{lemma:variance_score}
    Let $S(M) = \ln L_P(M) - \ln L_Q(M) = n[d_{\1W}(Q, P_M) - d_{\1W}(P, P_M)]$ be the score function. Under the hypotheses of \Cref{lemma:met_sens_lik}, the following two bounds hold.
    \begin{enumerate}
        \item[(i)] ( worst-case bound.)
        \begin{align}
            \label{eq:score_bound_rigorous}
            |S(M)| \leq nD\epsilon \quad \text{for $\Pi_P$-almost every } M, 
            \qquad \mathrm{Var}_{\Pi_P}(S(M)) \leq (nD\epsilon)^2.
        \end{align}
        \item[(ii)] (Leading-order representation.) Assume the posterior is supported in a $\delta_n$-neighborhood of the identifiability set $\{M : P_M = P\}$, parametrize $P_M = P + \delta_n \eta_M$ with $\|\eta_M\| = 1$, and set $\sigma_n^2 := \mathrm{Var}_{\Pi_P}(\langle\Delta, \eta_M\rangle)$ where $\Delta = (Q-P)/\|Q-P\|$. Then, to leading order in $\epsilon$,
        \begin{align}
            \label{eq:score_var_leading}
            \mathrm{Var}_{\Pi_P}(S(M)) = n^2 D^2 \epsilon^2 \, \sigma_n^2 + o(n^2 \epsilon^2).
        \end{align}
        Under the additional regularity conditions that imply the posterior contraction rate $\delta_n$ and the score variance $\sigma_n$ are of the same order, we get that $\mathrm{Var}_{\Pi_P}(S(M)) = O(n^2 D^2 \epsilon^2 \delta_n^2)$.
    \end{enumerate}
\end{lemma}

\begin{proof}
    \emph{Part (i).} The pointwise bound $|S(M)| \leq nD\epsilon$ is immediate from \Cref{lemma:met_sens_lik}. The variance bound then follows by writing $\mathrm{Var}(X) = \mathbb{E}[X^2] - (\mathbb{E}X)^2 \leq \mathbb{E}[X^2] \leq \|X\|_\infty^2$.

    \emph{Part (ii).} On the support of $\Pi_P$ we write $P_M = P + \delta_n \eta_M$ with $\|\eta_M\|=1$, and take $d_{\1W}(P,P_M)=\delta_n$.
    By \Cref{eq:W1_TV}, for  a perturbation with total variation $\epsilon$, we have that $d_{\1W}(P,Q)\le D\epsilon$; define $\epsilon':=D\epsilon$ and $\Delta:=(Q-P)/\|Q-P\|$, so $\|\Delta\|=1$ and $Q=P+\epsilon'\Delta$.

    For perturbations of this form, combining \citep[Prop.~8.5.6]{ambrosio2005gradient} with the Hilbert structure of the $W_2$ tangent space at $P$ \citep[Def.~8.4.1]{ambrosio2005gradient}, the leading term of $d_{\1W}(Q,P_M)^2$ is $\|\epsilon'\Delta-\delta_n\eta_M\|^2$. With some algebra,
    \begin{align}
        d_{\1W}(Q,P_M)^2
        = \|\epsilon'\Delta-\delta_n\eta_M\|^2 + o\bigl((\epsilon')^2+\delta_n^2\bigr)
        = \delta_n^2 - 2\epsilon'\delta_n\langle\Delta,\eta_M\rangle + O\bigl((\epsilon')^2\bigr),
        \label{eq:W2_expand}
    \end{align}
    Since $d_{\1W}(P,P_M)=\delta_n$ and $\epsilon'\ll\delta_n$, Taylor's formula $\sqrt{1+x}=1+\tfrac{x}{2}+O(x^2)$ applied to \eqref{eq:W2_expand} yields
    \begin{align}
        d_{\1W}(Q,P_M) = \delta_n - \epsilon'\langle\Delta,\eta_M\rangle + O\!\bigl((\epsilon')^2/\delta_n\bigr).
    \end{align}
    Hence
    \begin{align}
        S(M) = n\bigl[d_{\1W}(Q,P_M)-d_{\1W}(P,P_M)\bigr]
        = -n\epsilon'\langle\Delta,\eta_M\rangle + o(n\epsilon').
    \end{align}
    Taking variance over the posterior,
    \begin{align}
        \mathrm{Var}_{\Pi_P}(S(M)) = n^2(\epsilon')^2\, \mathrm{Var}_{\Pi_P}(\langle\Delta, \eta_M\rangle) + o(n^2(\epsilon')^2) = n^2 D^2 \epsilon^2 \, \sigma_n^2 + o(n^2\epsilon^2),
    \end{align}
    which is \Cref{eq:score_var_leading}. Under the regularity assumption stated, $\sigma_n = O(\delta_n)$, giving the claimed scaling.
\end{proof}

\begin{remark}
    \label{rem:score_var_regimes}
    The two parts of \Cref{lemma:variance_score} should be read as a hierarchy of bounds. Part (i) is rigorous and uniform in $\epsilon$ and $n$ but discards the contraction of the posterior. Part (ii), combined with $\sigma_n = O(\delta_n)$, captures this contraction and yields the asymptotically sharper bound used in \Cref{thm:posterior_stability}. $\delta_n$ captures how far $P_M$ is from $P$ on average, while $\sigma_n$ is angular sensitivity width (how much $\langle\Delta,\eta_M\rangle$ varies under $\Pi_P$). The additional regularity condition that $\sigma_n = O(\delta_n)$ is stated as a regularity condition for the contracting regime, but may not hold generally. 
\end{remark}

\begin{lemma}[Posterior Bound]
\label{lemma:pos_tv_bound}
    Under the hypotheses of \Cref{lemma:met_sens_lik}, in the small-perturbation regime $nD\epsilon \to 0$,
    \begin{align}
        d_{TV}(\Pi_P, \Pi_Q) \leq C \cdot n \cdot \sigma_n \cdot \epsilon \cdot \bigl(1 + O(nD\epsilon)\bigr) + o(n\epsilon),
    \end{align}
    where $\sigma_n$ is the posterior-geometry constant of \Cref{lemma:variance_score}(ii) and $C$ is an absolute constant. In the regular setting where $\sigma_n = O(\delta_n)$ (\Cref{rem:score_var_regimes}), this yields
    \begin{align}
        d_{TV}(\Pi_P, \Pi_Q) \leq C \cdot n \cdot \delta_n \cdot \epsilon \cdot \bigl(1 + O(nD\epsilon)\bigr) + o(n\epsilon).
    \end{align}
\end{lemma}

\begin{proof}
    We use Pinsker's inequality $d_{TV}(\Pi_P, \Pi_Q) \leq \sqrt{\tfrac{1}{2} d_{KL}(\Pi_P, \Pi_Q)}$ and bound the KL divergence via the cumulant generating function of the score.

    Let $\pi_P = L_P/Z_P$, $\pi_Q = L_Q/Z_Q$ be the Radon--Nikodym derivatives of the posteriors with respect to $\Pi_0$, and $S(M) = \ln L_P(M) - \ln L_Q(M)$. Substituting the definitions,
    \begin{align}
        d_{KL}(\Pi_P, \Pi_Q) 
        &= \int \pi_P(M) \left[ \ln \tfrac{L_P(M)}{L_Q(M)} + \ln \tfrac{Z_Q}{Z_P} \right] d\Pi_0(M)\\
        &= \mathbb{E}_{\Pi_P}[S(M)] + \ln \tfrac{Z_Q}{Z_P}.
    \end{align}
    By \Cref{lemma:stab_ev_ratio}, $Z_Q/Z_P = \mathbb{E}_{\Pi_P}[L_Q/L_P] = \mathbb{E}_{\Pi_P}[e^{-S(M)}]$, so
    \begin{align}
        \label{eq:KL_score}
        d_{KL}(\Pi_P, \Pi_Q) = \mathbb{E}_{\Pi_P}[S(M)] + \ln \mathbb{E}_{\Pi_P}\bigl[e^{-S(M)}\bigr] = R(-1),
    \end{align}
    where $K(t) := \ln \mathbb{E}_{\Pi_P}[e^{tS(M)}]$ is the cumulant generating function of $S$ under $\Pi_P$ and $R(t) := K(t) - t\, \mathbb{E}_{\Pi_P}[S]$ is the centered cumulant remainder.

    By \Cref{lemma:variance_score}(i), $|S(M)| \leq nD\epsilon =: B$ a.s.\ under $\Pi_P$, so $K(t)$ is finite and analytic on $\mathbb{R}$. Bennett's inequality for bounded random variables gives, for all $t \in \mathbb{R}$,
    \begin{align}
        R(t) \leq \frac{\mathrm{Var}_{\Pi_P}(S)}{B^2} \bigl( e^{tB} - tB - 1 \bigr).
    \end{align}
    Evaluating at $t = -1$ and using $e^{-B} + B - 1 = B^2/2 + O(B^3)$ as $B \to 0$,
    \begin{align}
        d_{KL}(\Pi_P, \Pi_Q) = R(-1) \leq \tfrac{1}{2}\,\mathrm{Var}_{\Pi_P}(S) \cdot \bigl( 1 + O(B) \bigr).
    \end{align}
    Substituting the leading-order representation $\mathrm{Var}_{\Pi_P}(S) = n^2 D^2 \epsilon^2 \sigma_n^2 + o(n^2\epsilon^2)$ from \Cref{eq:score_var_leading},
    \begin{align}
        d_{KL}(\Pi_P, \Pi_Q) \leq \tfrac{1}{2}\, n^2 D^2 \epsilon^2 \sigma_n^2 \bigl(1 + O(nD\epsilon)\bigr) + o(n^2 \epsilon^2).
    \end{align}
    Pinsker's inequality then yields
    \begin{align}
        d_{TV}(\Pi_P, \Pi_Q) \leq \tfrac{1}{2}\, nD\sigma_n \epsilon \cdot \bigl(1 + O(nD\epsilon)\bigr) + o(n\epsilon).
    \end{align}
    Absorbing $D/2$ into the constant $C$ gives the first claim. The second follows from $\sigma_n = O(\delta_n)$ in regular settings.
\end{proof}
    
\textbf{\Cref{thm:posterior_stability} restated.} Let $P, Q$ be two distributions on the data such that $d_{TV}(P, Q) \leq \epsilon$, and let $n$ be the number of samples. Then, for a well-defined prior $\Pi_0$ and in the small-perturbation regime $n\epsilon \to 0$,
\begin{align}
    d_{TV}(\Pi_P, \Pi_Q) \leq C \cdot n \cdot \delta_n \cdot \epsilon,
\end{align}
$\delta_n$ is the posterior contraction rate, and $C$ is an absolute constant.

\begin{proof}
    Direct from \Cref{lemma:pos_tv_bound}, using $\sigma_n = O(\delta_n)$ in regular settings (\Cref{rem:score_var_regimes}) and absorbing factors of $D$ as well as the $1 + O(nD\epsilon)$ correction into $C$.
\end{proof}

\begin{lemma}[Pushforward stability]
    \label{lemma:pushforward_stability}
    Under the conditions of \Cref{thm:posterior_stability}, if $d_{TV}(P, Q) \leq \epsilon$ with $nD\epsilon \to 0$, then
    \begin{align}
        d_{TV}(\Psi_\# \Pi_P, \Psi_\# \Pi_Q) \leq d_{TV}(\Pi_P, \Pi_Q) \leq C \cdot n \cdot \delta_n \cdot \epsilon.
    \end{align}
\end{lemma}
\begin{proof}
    By the data processing inequality applied to $\Psi : \3M \to \3R$, which is Borel measurable since $\Psi$ is continuous (\Cref{prop:continuity_causal_functional}),
    \begin{align}
        d_{TV}(\Psi_\# \Pi_P, \Psi_\# \Pi_Q) \leq d_{TV}(\Pi_P, \Pi_Q).
    \end{align}
    Indeed, $\Psi_\# \Pi_P(B) = \Pi_P(\Psi^{-1}(B))$ for $B \in \1B(\3R)$, and the family $\{\Psi^{-1}(B) : B \in \1B(\3R)\}$ is a sub-$\sigma$-algebra of the measurable sets in $\3M$, so the supremum defining $d_{TV}$ on $\3R$ is bounded by the supremum defining $d_{TV}$ on $\3M$.
    The second inequality is \Cref{thm:posterior_stability}.
\end{proof}

\textbf{\Cref{prop:stable_summaries} restated.}
Let $T$ be effect-aligned in the sense of \Cref{def:effect_aligned}, and assume $|\Psi(M)| \leq K$ for all $M \in \3M$.
Under the conditions of \Cref{thm:posterior_stability}, if $d_{TV}(P, Q) \leq \epsilon$ with $nD\epsilon \to 0$, then $|T(\Pi_P) - T(\Pi_Q)| \to 0$ as $\epsilon \to 0$.

\begin{proof}
    Write $T(\Pi) = E(\Psi_\#\Pi)$ as in \Cref{def:effect_aligned}. By \Cref{thm:posterior_stability} and \Cref{lemma:pushforward_stability}, $d_{TV}(\Psi_\# \Pi_P, \Psi_\# \Pi_Q) \to 0$ as $\epsilon \to 0$.
    On a compact interval containing $\mathrm{supp}(\Psi_\# \Pi_P) \cup \mathrm{supp}(\Psi_\# \Pi_Q)$, total-variation convergence implies weak convergence of $\Psi_\# \Pi_P$ and $\Psi_\# \Pi_Q$.
    Weak continuity of $E$ then yields $T(\Pi_P) = E(\Psi_\# \Pi_P) \to E(\Psi_\# \Pi_Q) = T(\Pi_Q)$.
\end{proof}

\subsection{Proofs of statements in \Cref{sec:mode}}
\label{app:frequentist_scm}
In this section we prove that IPW, outcome regression, and doubly robust estimators can be interpreted as first selecting an SCM via maximum likelihood (or equivalent criteria) and then evaluating the causal effect under that SCM.

\begin{proposition}[IPW as SCM selection]
\label{prop:ipw_scm}
Let $\mathcal{E}$ be a parametric class of propensity models $e_\theta(x|z)$. The IPW estimator
\begin{align}
    \hat{\Psi}_{\mathrm{IPW}} = \frac{1}{n}\sum_{i=1}^n \frac{Y_i \I\{X_i=x\}}{\hat{e}(X_i\mid Z_i)}, \quad \text{where } \hat{e} \in \arg\max_{e \in \mathcal{E}} L_P(e),
\end{align}
equals $\Psi_P(M_{\hat{e}})$ for any SCM $M_{\hat{e}} \in \3M$ whose treatment mechanism is $\hat{e}$ and whose observational distribution $P_{M_{\hat{e}}}$ equals $P$ on the support of the data.
\end{proposition}
\begin{proof}
The propensity $\hat{e}$ is selected by maximum likelihood: $\hat{e} \in \arg\max_{e} \prod_i e(X_i|Z_i)$. Consider an SCM $M$ with treatment mechanism $f_X(Z, U_X)$ inducing $P_M(X|Z) = \hat{e}(X|Z)$, and any outcome mechanism compatible with $P$. By the second form of the back-door formula:
\begin{align}
    \Psi(M) = \3E_M\left[\frac{Y \I\{X=x\}}{\hat{e}(X\mid Z)}\right] = \3E_P\left[\frac{Y \I\{X=x\}}{\hat{e}(X\mid Z)}\right] = \hat{\Psi}_{\mathrm{IPW}},
\end{align}
where the second equality holds because $P_M = P$ on observables under consistency. The sample mean is the empirical counterpart of $\3E_P[\cdot]$.
\end{proof}

\begin{proposition}[Outcome regression as SCM selection]
\label{prop:regression_scm}
Let $\mathcal{M}_\mu$ be a parametric class of outcome models $\mu_\theta(x,z)$. The regression estimator
\begin{align}
    \hat{\Psi}_{\mathrm{reg}} = \frac{1}{n}\sum_{i=1}^n \hat{\mu}(x, Z_i), \quad \text{where } \hat{\mu} \in \arg\max_{\mu \in \mathcal{M}_\mu} L_P(\mu),
\end{align}
equals $\Psi_P(M_{\hat{\mu}})$ for any SCM $M_{\hat{\mu}} \in \3M$ whose outcome mechanism satisfies $\3E_M[Y|X,Z] = \hat{\mu}(X,Z)$ and whose observational distribution matches $P$.
\end{proposition}
\begin{proof}
The outcome model $\hat{\mu}$ is selected by maximum likelihood (e.g., Gaussian OLS or logistic regression). By the first form of the back-door formula:
\begin{align}
    \Psi(M) = \int \hat{\mu}(x,z) P(z) dz = \3E_P[\hat{\mu}(x,Z)] = \hat{\Psi}_{\mathrm{reg}},
\end{align}
using $P_M = P$ for the marginal of $Z$. The sample mean over $Z_i$ is the empirical estimate of $\3E_P[\hat{\mu}(x,Z)]$.
\end{proof}

\begin{proposition}[Doubly robust as SCM selection]
\label{prop:dr_scm}
The doubly robust estimator, which uses both $\hat{e}$ and $\hat{\mu}$ (each from MLE), can be written as $\hat{\Psi}_{\mathrm{DR}} = n^{-1}\sum_i \phi(Y_i,X_i,Z_i;\hat{e},\hat{\mu})$ where
$\phi(y,x,z;e,\mu) = \mu(x,z) + \frac{(y-\mu(x,z))\I\{X=x\}}{e(x\mid z)}$. This equals $\Psi_P(M)$ for an SCM $M$ whose mechanisms are $(\hat{e}, \hat{\mu})$, since $\3E[\phi(Y,X,Z;\hat{e},\hat{\mu})] = \Psi_P(M)$ under either correct propensity or correct outcome model, and both are selected by $\arg\max$ over their respective spaces.
\end{proposition}

\begin{proof}
The DR estimator is consistent if either $\hat{e}$ or $\hat{\mu}$ is consistent. The identifying expression $\3E[\phi] = \Psi$ holds for any SCM with mechanisms $(e,\mu)$ such that the observed data distribution is compatible. The selection step picks $(\hat{e}, \hat{\mu})$ via separate MLEs; the transformation step evaluates $\Psi(M_{(\hat{e},\hat{\mu})})$ via the DR formula. Thus the two-step interpretation holds.
\end{proof}

\begin{lemma}[Fragility of causal identification]
\label{lem:fragility}
Let $\G$ be a causal diagram with $X$ an ancestor of $Y$. Let $P$ be any observational distribution over $\V$ with density $p$. Assume imperfect overlap: there exists a set $Z_0 \subseteq \Omega_Z$ with $P(Z \in Z_0) > 0$ such that $p(x,z)$ is arbitrarily small on $Z_0$. Then for any $\delta > 0$ and $\Delta > 0$, there exist SCMs $M_1, M_2 \in \3M$ Markov to $\G$ such that (i) $d_{\mathrm{TV}}(P_{M_1}, P_{M_2}) < \delta$, and (ii) $|\Psi(M_1) - \Psi(M_2)| \geq \Delta$.
\end{lemma}
\begin{proof}
Without loss of generality, assume the graph $Z \to X$, $\{X,Z\} \to Y$ (any diagram with $X$ an ancestor of $Y$ admits a back-door through some set $Z$). Let $p(x,z)$ denote the marginal density of $(X,Z)$ under $P$, and let $\Psi = \3E[Y \mid do(X=x)]$ be the target causal effect.

\textbf{Construction.} Fix $f \colon \Omega_X \times \Omega_Z \to \3R$ bounded, and $\sigma^2 > 0$. Define two SCMs with the same treatment and confounder mechanisms as $P$, and outcome equations:
\begin{align}
    M_1: \quad Y &= f(X, Z) + \epsilon, \quad \epsilon \sim \mathcal{N}(0, \sigma^2), \\
    M_2: \quad Y &= f(X, Z) + A \cdot g_k(X, Z) + \epsilon, \quad \epsilon \sim \mathcal{N}(0, \sigma^2),
\end{align}
where $g_k(x,z) = \exp(-k \cdot p(x,z))$ for $k > 0$ and $A \in \3R$ is chosen below.

\textbf{Observational closeness.} The observational density under $M_i$ factors as $p_{M_i}(x,z,y) = p(x,z) \cdot p_{M_i}(y \mid x,z)$. The conditional under $M_1$ is $\mathcal{N}(f(x,z), \sigma^2)$; under $M_2$ it is $\mathcal{N}(f(x,z) + A g_k(x,z), \sigma^2)$. On the set $\{(x,z) : p(x,z) \geq \eta\}$ for any $\eta > 0$, we have $g_k(x,z) \leq \exp(-k\eta) \to 0$ as $k \to \infty$. Thus $p_{M_2}(y \mid x,z) \to p_{M_1}(y \mid x,z)$ pointwise on a set of measure arbitrarily close to 1. By dominated convergence, $d_{\mathrm{TV}}(P_{M_1}, P_{M_2}) \to 0$ as $k \to \infty$. Hence for any $\delta > 0$, we can choose $k$ large enough so that $d_{\mathrm{TV}}(P_{M_1}, P_{M_2}) < \delta$.

\textbf{Causal divergence.} By the back-door formula, $\Psi(M) = \int \3E_M[Y \mid X=x, Z=z] \, p(z) dz$. Thus
\begin{align}
    \Psi(M_2) - \Psi(M_1) = A \int g_k(x,z) \, p(z) dz =: A \cdot \alpha_k.
\end{align}
By imperfect overlap, there exists $Z_0$ with $\alpha := P(Z \in Z_0) > 0$ and $p(x,z)$ arbitrarily small on $Z_0$. On $Z_0$, $g_k(x,z) \to 1$ as $k \to \infty$; on the complement, $g_k \leq 1$. By bounded convergence, $\alpha_k \to \int_{Z_0} p(z) dz = \alpha > 0$ as $k \to \infty$. For $k$ large, $\alpha_k \geq \alpha/2 > 0$. Setting $A = 2\Delta/\alpha$ yields $|\Psi(M_2) - \Psi(M_1)| \geq \Delta$.
\end{proof}

\textbf{\Cref{thm:model_selective_discontinuity} restated.}
Let $T$ be model-selective in the sense of \Cref{def:model_selective}.
    Let $\3M$ be a nonparametric hypothesis class of SCMs over continuous variables with $X\in An_{\G}(Y)$.
    Under imperfect overlap, there exist observational distributions $P$ and sequences $Q_n \to P$ in $d_{TV}$ such that
    \[
        \limsup_{n\to\infty} \ \bigl| \ T(\Pi_{Q_n}) - T(\Pi_P) \ \bigr| \geq \Delta
    \]
    for some $\Delta > 0$.
    In particular, the map $P \mapsto T(\Pi_P)$ is not continuous at such $P$.

\begin{proof}
We exhibit, for any $P$ and any $\Delta > 0$, a sequence of distributions $Q_n \to P$ (in $d_{\mathrm{TV}}$) such that $|\Psi(M_n^*) - \Psi(M^*)| \geq \Delta$ for some choices $M_n^* \in \arg\max_{M} L_{Q_n}(M)\Pi_0(M)$ and $M^* \in \arg\max_{M} L_P(M)\Pi_0(M)$.

By \Cref{lem:fragility}, fix $\delta > 0$ arbitrarily small and choose $M_1, M_2 \in \3M$ such that $d_{\mathrm{TV}}(P_{M_1}, P_{M_2}) < \delta$ and $|\Psi(M_1) - \Psi(M_2)| \geq \Delta$. Set $P = P_{M_1}$ and $Q = P_{M_2}$. Then $d_{\mathrm{TV}}(P, Q) < \delta$.

Under $P = P_{M_1}$, the likelihood $L_P(M_1)$ is maximal (the true model). Under $Q = P_{M_2}$, the likelihood $L_Q(M_2)$ is maximal. Assume the prior $\Pi_0$ assigns positive mass to both $M_1$ and $M_2$ (generic for non-parametric classes). Then:
\begin{itemize}
    \item For the selection under $P$: $M^* \in \arg\max L_P(M)\Pi_0(M)$ includes $M_1$, so we may take $M^* = M_1$ and $\Psi(M^*) = \Psi(M_1)$.
    \item For the selection under $Q$: $M^*_Q \in \arg\max L_Q(M)\Pi_0(M)$ includes $M_2$, so we may take $M^*_Q = M_2$ and $\Psi(M^*_Q) = \Psi(M_2)$.
\end{itemize}
Thus $|\Psi(M^*_Q) - \Psi(M^*)| = |\Psi(M_2) - \Psi(M_1)| \geq \Delta$, while $d_{\mathrm{TV}}(P, Q) < \delta$. Taking $\delta \to 0$ yields a sequence $Q_n \to P$ with $|\Psi(M^*_{Q_n}) - \Psi(M^*)| \geq \Delta$, so the map $P \mapsto \hat\Psi(P)$ is discontinuous at $P$.
\end{proof}

\subsection{Proofs of statements in \Cref{sec:mean}}

\textbf{\Cref{cor:bound_posterior_mean} restated.} Let $\Psi : \3M \to \3R$ be defined by $\Psi(M) = \3E_{P_M}[Y_x]$, and assume $\sup_M |\Psi(M)| \leq K$. Under the hypotheses of \Cref{thm:posterior_stability},
\begin{align}
    \Big| \, \mathbb{E}_{\Pi_P}[\Psi] - \mathbb{E}_{\Pi_Q}[\Psi] \, \Big| \leq 2K \cdot C \cdot n \cdot \delta_n \cdot \epsilon.
\end{align}

\begin{proof}
    Combining \Cref{lemma:exp_diff_bound} and \Cref{lemma:pos_tv_bound},
    \begin{align}
        \Big| \, \mathbb{E}_{\Pi_P}[\Psi] - \mathbb{E}_{\Pi_Q}[\Psi] \, \Big| 
        \leq 2K \cdot d_{TV}(\Pi_P, \Pi_Q) 
        \leq 2K \cdot C \cdot n \cdot \delta_n \cdot \epsilon.
    \end{align}
    Equivalently, the map $P \mapsto \mathbb{E}_{\Pi_P}[\Psi]$ is Lipschitz in $\epsilon$ with constant $L_\Psi = O(n\delta_n)$, so the error in the causal estimate vanishes as $\epsilon \to 0$.
\end{proof}

\begin{theorem}[Posterior Variance-Sensitivity Bound]
    \label{thm:posterior_variance_sensitivity_app}
    Let $S(M) = \ln L_P(M) - \ln L_Q(M)$ be the score function and let $\{\Pi_t\}_{t \in [0,1]}$ be the geometric tilting path with $\Pi_0 = \Pi_P$, $\Pi_1 = \Pi_Q$, and Radon--Nikodym derivative
    \begin{align}
        \frac{d\Pi_t}{d\Pi_P}(M) = \frac{e^{-tS(M)}}{Z_t}, \qquad Z_t := \mathbb{E}_{\Pi_P}\bigl[e^{-tS(M)}\bigr].
    \end{align}
    Then the deviation in the posterior mean of the causal effect between $P$ and $Q$ is bounded by
    \begin{align}
        \Big|\3E_{\Pi_P}[\Psi] - \3E_{\Pi_Q}[\Psi]\Big| 
        &\leq \int_0^1 \sqrt{\mathrm{Var}_{\Pi_t}(\Psi(M))}\, \sqrt{\mathrm{Var}_{\Pi_t}(S(M))}\, dt \\
        &\leq \sup_{t \in [0,1]} \sqrt{\mathrm{Var}_{\Pi_t}(\Psi(M))}\, \sqrt{\mathrm{Var}_{\Pi_t}(S(M))}.
    \end{align}
    To leading order in $\epsilon$, the right-hand side equals $\sqrt{\mathrm{Var}_{\Pi_P}(\Psi(M))}\, \sqrt{\mathrm{Var}_{\Pi_P}(S(M))} + o(\epsilon)$.
\end{theorem}

\begin{proof}
    \emph{Path connecting the posteriors.} At $t = 0$, $\Pi_0 = \Pi_P$ trivially, and at $t = 1$, $L_Q/L_P = e^{-S}$ together with $Z_1 = Z_Q/Z_P$ (\Cref{lemma:stab_ev_ratio}) gives
    \begin{align}
        \frac{d\Pi_1}{d\Pi_P}(M) = \frac{e^{-S(M)}}{Z_Q/Z_P} = \frac{L_Q(M)/L_P(M)}{Z_Q/Z_P} = \frac{L_Q(M) / Z_Q}{L_P(M) / Z_P} = \frac{d\Pi_Q}{d\Pi_P}(M),
    \end{align}
    so $\Pi_1 = \Pi_Q$.

    \emph{Differentiating along the path.} Set $\bar\Psi(t) := \mathbb{E}_{\Pi_t}[\Psi(M)] = Z_t^{-1}\, \mathbb{E}_{\Pi_P}[\Psi(M) e^{-tS(M)}]$. By the quotient rule and dominated convergence (justified by $|S(M)| \leq nD\epsilon$ a.s., \Cref{lemma:variance_score}(i)),
    \begin{align}
        \frac{d \bar\Psi(t)}{dt} 
        &= \frac{-\mathbb{E}_{\Pi_P}[\Psi(M) S(M) e^{-tS(M)}]}{Z_t} - \frac{\mathbb{E}_{\Pi_P}[\Psi(M) e^{-tS(M)}]}{Z_t} \cdot \frac{-\mathbb{E}_{\Pi_P}[S(M) e^{-tS(M)}]}{Z_t}\\
        &= -\mathbb{E}_{\Pi_t}[\Psi(M) S(M)] + \mathbb{E}_{\Pi_t}[\Psi(M)] \cdot \mathbb{E}_{\Pi_t}[S(M)]\\
        &= -\mathrm{Cov}_{\Pi_t}\bigl(\Psi(M), S(M)\bigr).
    \end{align}

    \emph{Integration and Cauchy--Schwarz.} Integrating along the path,
    \begin{align}
        \mathbb{E}_{\Pi_Q}[\Psi] - \mathbb{E}_{\Pi_P}[\Psi] = \bar\Psi(1) - \bar\Psi(0) = -\int_0^1 \mathrm{Cov}_{\Pi_t}\bigl(\Psi(M), S(M)\bigr)\, dt.
    \end{align}
    By Cauchy--Schwarz, $|\mathrm{Cov}_{\Pi_t}(\Psi, S)| \leq \sqrt{\mathrm{Var}_{\Pi_t}(\Psi)}\sqrt{\mathrm{Var}_{\Pi_t}(S)}$, and bounding the integral by its supremum gives the claimed bound.

    \emph{Leading-order behavior.} Along the path, $\Pi_t$ deviates from $\Pi_P$ by an $O(t \cdot nD\epsilon) = O(\epsilon)$ tilt in the relative density, so $\mathrm{Var}_{\Pi_t}(\Psi)$ and $\mathrm{Var}_{\Pi_t}(S)$ deviate from their values at $\Pi_P$ by $O(\epsilon)$ corrections. Hence
    \begin{align}
        \sup_{t \in [0,1]} \sqrt{\mathrm{Var}_{\Pi_t}(\Psi)}\sqrt{\mathrm{Var}_{\Pi_t}(S)} = \sqrt{\mathrm{Var}_{\Pi_P}(\Psi)}\sqrt{\mathrm{Var}_{\Pi_P}(S)} + o(\epsilon),
    \end{align}
    proving the leading-order claim.
\end{proof}

\begin{remark}
    Combining \Cref{thm:posterior_variance_sensitivity_app} with \Cref{lemma:variance_score}(ii) and the Lipschitz continuity of $\Psi$ on $\3M$ (\Cref{prop:continuity_causal_functional}, which gives $\sqrt{\mathrm{Var}_{\Pi_P}(\Psi(M))} = O(\delta_n)$ under posterior contraction at rate $\delta_n$), one recovers the heuristic scalings $\sqrt{\mathrm{Var}_{\Pi_P}(\Psi(M))} \sim \delta_n$ and $\sqrt{\mathrm{Var}_{\Pi_P}(S(M))} \sim n D \delta_n \epsilon$ stated in \Cref{sec:mean}, and hence the bound
    \begin{align}
        \Big|\3E_{\Pi_P}[\Psi] - \3E_{\Pi_Q}[\Psi]\Big| \leq C \cdot n \delta_n^2 \epsilon + o(n \delta_n^2 \epsilon),
    \end{align}
    for an absolute constant $C$. This is sharper than \Cref{cor:bound_posterior_mean} by a factor of $\delta_n$ in the regime where the posterior contracts.
\end{remark}


\end{document}